\documentclass[journal]{IEEEtran}
\usepackage{graphicx}
\input{psfig.sty}
\input{epsf.sty}
\usepackage{amsmath}
\usepackage{amssymb}
\usepackage{amsfonts}
\usepackage{amscd}
\usepackage{mathrsfs}
\usepackage{psfrag}
\usepackage{caption}
\usepackage{subcaption}
\usepackage{breqn}
\usepackage{algorithm,caption}
\usepackage[noend]{algpseudocode}
\usepackage{color}

\newcommand{\mbx}{\mathbf{x}}
\newcommand{\mbA}{\mathbf{A}}
\newcommand{\mbH}{\mathbf{H}}
\newcommand{\mby}{\mathbf{y}}
\newcommand{\mbe}{\mathbf{e}}

\newcommand{\T}{\mathcal{T}}

\newcommand{\supp}{\ensuremath{\texttt{supp}}}

\newtheorem{theorem}{Theorem}
\newtheorem{lemma}{Lemma}

\newtheorem{corollary}{Corollary}

\newtheorem{assumption}{Assumption}

\newenvironment{proof}[1][Proof]{\begin{trivlist}
\item[\hskip \labelsep {\bfseries #1}]}{\end{trivlist}}

\newcommand{\qed}{\nobreak \ifvmode \relax \else
      \ifdim\lastskip<1.5em \hskip-\lastskip
      \hskip1.5em plus0em minus0.5em \fi \nobreak
      \vrule height0.75em width0.5em depth0.25em\fi}

\newcommand*{\QEDA}{\hfill\ensuremath{\blacksquare}}%
\newcommand*{\QEDB}{\hfill\ensuremath{\square}}%

\begin{document}
%
\title{Estimate Exchange over Network is Good for \\ Distributed Hard Thresholding Pursuit}


%
\author{\IEEEauthorblockN{Ahmed Zaki\IEEEauthorrefmark{1}, Partha P. Mitra\IEEEauthorrefmark{2}, Lars K. Rasmussen\IEEEauthorrefmark{1} and Saikat Chatterjee\IEEEauthorrefmark{1}} \\
\IEEEauthorblockA{\IEEEauthorrefmark{1} School of Electrical Engineering, KTH Royal Institute of Technology, Stockholm, Sweden} \\
\IEEEauthorblockA{\IEEEauthorrefmark{2} Cold Spring Harbor Laboratory, 1 Bungtown Road, New York, USA}}


\maketitle

\begin{abstract}
We investigate an existing distributed algorithm for learning sparse signals or data over networks. The algorithm is iterative and exchanges intermediate estimates of a  sparse signal over a network. This learning strategy using exchange of  intermediate estimates over the network requires a limited communication overhead for information transmission. Our objective in this article is to show that the strategy is good for learning in spite of limited communication. In pursuit of this objective, we first provide a restricted isometry property (RIP)-based theoretical analysis on convergence of the iterative algorithm. Then, using simulations, we show that the algorithm provides competitive performance in learning sparse signals vis-a-vis an existing alternate distributed algorithm. The alternate distributed algorithm exchanges more information including observations and system parameters.  
\end{abstract}


%
\IEEEpeerreviewmaketitle

\section{Introduction}
The topic of estimating and/or learning of sparse signals has many applications, such as machine learning \cite{Tipping_2001_RVM, Sastry_2009_Face_recognition}, multimedia processing \cite{Elad_book_2010}, compressive sensing \cite{signal_processing_article_CS_2008}, wireless communications \cite{Schepker_sparse_MUD_2011}, just to mention a few.
Here we consider a distributed sparse learning problem.
Consider a network consisting of $L$ nodes with a network matrix $\mbH \in \mathbb{R}^{L \times L}$ describing the connections among the nodes. The $(l,r)$'th element $h_{lr}$ of $\mbH$ specifies the weight of the link from node $r$ to node $l$. A zero valued $h_{lr}$ signifies the absence of a direct link from node $r$ to node $l$. 
In the literature, $\mbH$ is also known as a network policy matrix.
Let $\mbx \in \mathbb{R}^N$ be the underlying sparse signal to estimate/learn. Further assume that node $l$ has the observation 
\begin{eqnarray}
\label{eq:system_model}
\mby_l = \mbA_l \mbx + \mbe_l.
\end{eqnarray}
Here $\mbe_l$ is the error term and $\mbA_l \in \mathbb{R}^{M_l \times N}$ is the system matrix (a sensing matrix or dictionary, depending on a particular application) at node $l$.  For distributed learning, the nodes of a network exchange various information, for example, intermediate estimates of $\mbx$, observations $\mby_l$, system matrices $\mbA_l$, or their parameters. Using such information, a distributed algorithm learns $\mbx$ at each node over iterations. Few important aspects of a distributed algorithm are scalability with the number of nodes, low computational requirements at each node, and limited communication between nodes. These aspects are required to realize a large distributed system with many nodes and high-dimensional system matrices. To comply with the low computational aspect, we focus on greedy algorithms. Standard greedy algorithms, such as orthogonal matching pursuit \cite{Tropp_2007_OMP}, subspace pursuit \cite{Dai_2009_Subspace_pursuit}, CoSaMP \cite{Needell_Tropp_2008_CoSaMP}, and their variants \cite{BacktrackingMP_Makur_2011, Chatterjee_Sundman_Vehkapera_Skoglund_TSP_2012} are of low computational load and fast in execution. Further, for limited communication aspect, we prefer that intermediate estimates of $\mbx$ are exchanged over the network. Therefore we focus on developing distributed greedy sparse learning algorithms that exchange intermediate estimates of $\mbx$ or relevant parameters over network. Relevant past works in this direction are \cite{Zaki_Venkitaraman_Chatterjee_Rasmussen_GreedySparseLearningOverNetwork_TSIPN_2017, Sundman_Chatterjee_Skoglund_2016, Sundman_Chatterjee_Skoglund_DistributedGreedyPursuits_2013} where we proposed some rules for information exchange over network and developed new greedy algorithms for various signal models.

In this current article, we investigate an existing distributed greedy algorithm in the literature \cite{Sergios_greedy_sparsity_algorithm_distributed_learning_TSP_2015}. Our interest is to analyze performance of the algorithm. We do not propose any new rule or new algorithm. The existing algorithm of \cite{Sergios_greedy_sparsity_algorithm_distributed_learning_TSP_2015} exchanges intermediate estimates of $\mbx$ between the nodes of the network. We refer to this algorithm as the distributed hard thresholding pursuit (DHTP). DHTP has not received due attention in \cite{Sergios_greedy_sparsity_algorithm_distributed_learning_TSP_2015} -- minor simulation results were reported and no theoretical analysis was performed.
For the DHTP, our main objective is to show that estimate-exchange is a good strategy for achieving a good learning performance. In pursuit of this objective, we provide theoretical analysis and extensive simulation results. Using simulations, we show that there is no significant incentive in performance gain due to exchange of observations $\mby_l$ and system matrices $\mbA_l$. Exchange of intermediate estimates over network is good and that helps in communication and computation constrained scenarios. 

\subsection{Literature survey}

We provide a literature survey for the problem of sparse learning over network. For this problem, learning by consensus is a popular strategy. A consensus strategy achieves estimates that are all the same at network nodes after convergence. That means $\forall l, \,\, \hat{\mbx}_l = \hat{\mbx}$, where $\hat{\mbx}_l$ denotes the signal estimate at node $l$ at convergence of a distributed sparse learning algorithm. Using a greedy learning approach, a consensus seeking algorithm was recently proposed in \cite{Sergios_greedy_sparsity_algorithm_distributed_learning_TSP_2015}. This algorithm is referred to as the distributed hard thresholding (DiHaT). To achieve consensus, the DiHaT exchanges intermediate estimates of $\mbx$, observations $\mby_l$ and system matrices $\mbA_l$. Furthermore, DiHaT uses a consensus seeking network matrix $\mbH$ with special properties; DiHaT requires that $\mbH$ be a doubly stochastic matrix. Here consensus seeking means that each and every node have estimates that are same at convergence.
Alternatively, without seeking consensus, there exist several greedy sparse learning algorithms. DHTP of  \cite{Sergios_greedy_sparsity_algorithm_distributed_learning_TSP_2015} is one example. Other examples include algorithms from \cite{Eldar_DCS_static_time_varying_network_TSP_2014, Eldar_mod_DIHT_ICASSP_2015,Sundman_Chatterjee_Skoglund_2016, Varshaney_OMP_joint_sparsity_pattern_recovery_TSP_2014, saikat_DIPP_TSP_2016}.

For distributed sparse learning, there exist several convex optimization based algorithms, mainly in the application area of distributed compressed sensing \cite{Duarte_Baraniuk_2005_Distributed_CS_Asilomar, Baron_DCS_2009}. Some of the algorithms provide a centralized solution using a distributed convex optimization algorithm called the alternating-direction-method-of-multipliers (ADMM) \cite{Boyd_ADMM_article_2011}. ADMM based distributed learning algorithms were proposed in \cite{Mota_Distributed_Basis_Pursuit_TSP_2012, Giannakis_Distributed_sparse_linear_regression_TSP_2010}. Specifically, the ADMM based method of \cite{Giannakis_Distributed_sparse_linear_regression_TSP_2010} is called D-LASSO. It is worth mentioning that the D-LASSO is shown to provide a slower convergence compared to greedy DiHaT in \cite{Sergios_greedy_sparsity_algorithm_distributed_learning_TSP_2015}. 
Using adaptive signal processing techniques such as gradient search, distributed sparse learning and sparse regression were realized in \cite{Sergios_adaptive_algorithm_distributed_learning_TSP_2012, Li_Distributed_RLS_over_networks_TSP_2014, Sayed_sparse_distributed_learning_diffusion_TSP_2013}. These adaptive algorithms typically use a mean-square-error cost averaged over all the nodes in a network to find an optimal solution via gradient search. Distributed learning and regression are then performed via diffusion of information over a network and adaptation in all individual nodes. Using a Bayesian framework for finding the posterior with sparsity promoting priors, a distributed-message-passing based method was proposed in \cite{Eldar_DAMP_GlobalSIP_2014} and a sparse Bayesian learning based method was proposed in \cite{DABMP_CS_Naffouri_TSP_2013}. Further, to promote sparsity in solutions, distributed system learning such as distributed dictionary learning was also considered in \cite{Kreutz_Delgado_2003, Liang_distributed_DL_in_sensor_networks_2014}. Next, we mention that there exist several signal models in the literature where sparse signals are not the same for all nodes of a network. For example, denoting the signal at node $l$ by $\mbx_l$, supports of $\mbx_l$ are the same in \cite{Varshaney_OMP_joint_sparsity_pattern_recovery_TSP_2014, Varshaney_wireless_CS_distributed_sparse_TSIPN_2015, Fosson_Distributed_recovery_joint_sparse_signals_TSP_2016, Fosson_distributed_ADMM_TSIPN_2015}, but not their signal values; further, $\mbx_l$ have common and private support and/or signal parts in \cite{Chen_decentralized_bayesian_DCS_TWC_2016, Sundman_Chatterjee_Skoglund_2016, saikat_DIPP_TSP_2016}.
In this article, we consider the setup \eqref{eq:system_model} where $\forall l$, $\mbx_l = \mbx$. 

\subsection{Contributions}
Our objective is to show that signal estimate exchange is good for distributed sparse learning. There is no need to exchange $\mby_l$ and $\mbA_l$. In pursuit of this objective, we investigate DHTP and our contributions are as follows. 
\begin{enumerate}
\item We provide a restricted-isometry-property (RIP) based theoretical analysis and convergence guarantee for DHTP. For error-free condition, that means $\forall l, \mathbf{e}_l = \mathbf{0}$, learned estimates at all nodes converge to the true signal $\mathbf{x}$ under some technical conditions. 
\item Using simulations, we show instances where DHTP provides better learning performance than DiHaT. For a fair comparison, we evaluate practical performance using doubly stochastic $\mbH$ matrix. 
\item We show that DHTP performs good for a general network matrix, not necessarily a doubly stochastic matrix.
\end{enumerate}

\subsection{Notation} 
Support-set $\mathcal{T}$ of $\mbx = [x_1 \, x_2 \, \hdots ]^{\top}$ is defined as $\mathcal{T} = \{ i: x_i \neq 0 \}$. We use $|\mathcal{T}|$ and $\mathcal{T}^{c}$ to denote the cardinality and complement of the set $\mathcal{T}$, respectively. For a matrix $\mbA \in \mathbb{R}^{M \times N}$, a sub-matrix $\mbA_{\mathcal{T}} \in \mathbb{R}^{M \times |\mathcal{T}|}$ consists of the columns of $\mbA$ indexed by $i \in \mathcal{T}$. Similarly, for $\mbx \in \mathbb{R}^{N}$, a sub-vector $\mbx_{\mathcal{T}}\in \mathbb{R}^{|\mathcal{T}|}$ is composed of the components of $\mbx$ indexed by $i \in \mathcal{T}$. Also we denote $(.)^{\top}$ and $(.)^{\dag}$ as transpose and pseudo-inverse, respectively. We define the function ${\supp}(\mathbf{x}, s) \triangleq \{$the set of indices corresponding to the $s$ largest amplitude components of $\mathbf{x} \}$. We use $\hat{\mbx}_{l,k}$ to denote the signal estimate at node $l$ and iteration $k$.

\section{DHTP Algorithm and Theoretical Analysis}

\begin{algorithm}[t]
\caption{Distributed HTP - at node $l$} \label{algo:DHTP}
\emph{Input}: $\mby_l$, $\mbA_l$, $s$, $\{ h_{lr} \}$ \\
\emph{Initialization}: $k \gets 0$; $\hat{\mbx}_{l,k} \gets \mathbf{0}$ (Estimate at $k$'th iteration)
\emph{Iteration}:
\begin{algorithmic}
\State \textbf{repeat}
\State $k \gets k+1$ \hfill (Iteration counter)
\end{algorithmic}
\begin{algorithmic}[1]
\State $\tilde{\T}_{l,k} \gets \supp (\hat{\mbx}_{l,k-1} + \mbA^{\top}_{l}(\mby_l - \mbA_l\hat{\mbx}_{l,k-1}), s)$
\State $\tilde{\mbx}_{l,k} \,\,\ \text{such that} \,\,\ \tilde{\mbx}_{\tilde{\T}_{l,k}} \gets \mbA^{\dag}_{l,\tilde{\T}_{l,k}} \mby_l \ ; \,\ \tilde{\mbx}_{\tilde{\T}^{c}_{l,k}} \gets \mathbf{0}$
\State $\check{\mbx}_{l,k} =  \sum\limits_{r \in \mathcal{N}_{l}} h_{lr} \,\ \tilde{\mbx}_{r,k}$ \label{step:information_fusion_dhtp}
\State $\hat{\T}_{l,k} \gets \supp (\check{\mbx}_{l,k},s)$
\State $\hat{\mbx}_{l,k} \,\,\ \text{such that} \,\,\ \hat{\mbx}_{\hat{\T}_{l,k}} \gets \check{\mbx}_{\hat{\T}_{l,k}} \ ; \,\ \hat{\mbx}_{\hat{\T}^{c}_{l,k}} \gets \mathbf{0}$
\end{algorithmic}
\begin{algorithmic}
\State \textbf{until} \emph{stopping criterion}
\end{algorithmic}
\emph{Output}: $\hat{\mbx}_{l}, \ \hat{\mathcal{T}}_{l} $
\end{algorithm}

The pseudo-code of the DHTP algorithm is shown in Algorithm~\ref{algo:DHTP}. 
In every iteration $k$, the nodes use standard algorithmic steps of Hard Thresholding Pursuit algorithm (HTP) \cite{Foucart_HTP_algorithm_SIAM_2011} along-with an extra step to include information about the estimates at the neighbors to refine the local estimate (see Step~\ref{step:information_fusion_dhtp} of Algorithm~\ref{algo:DHTP}). Here, we denote the neighborhood of node $l$ by $\mathcal{N}_l$, i.e. $\mathcal{N}_l \triangleq \{r : h_{lr} \neq 0\}$. For theoretical analysis, we use the standard definition of RIP of a matrix as given in \cite{Candes_Tao_2005}. We denote the $s$-Restricted Isometry Constant (RIC) of a matrix by $\delta_{s}$. We use $\|.\|$ and $\|.\|_0$ to denote the standard $\ell_2$ and $\ell_0$ norm of a vector, respectively. 
Throughout the paper unless specified, we have the following assumptions.
\begin{assumption}
The network matrix $\mbH$ is a right stochastic matrix. This assumption is quite general as any non-negative matrix can be reduced to a right stochastic matrix by appropriately scaling the rows of the matrix.
\end{assumption}
\begin{assumption}
The sparsity level of the signal $\mbx$, denoted by $s \triangleq \|\mbx\|_0$ is known a-priori. This assumption is used in the greedy algorithms such as CoSaMP\cite{Needell_Tropp_2008_CoSaMP}, subspace pursuit \cite{Dai_2009_Subspace_pursuit}, HTP \cite{Foucart_HTP_algorithm_SIAM_2011}, etc.
\end{assumption}

We first provide a recurrence inequality for DHTP, which provides performance bounds of the algorithm over iterations. For notational clarity, we define RIC constant 
\begin{equation*}
\delta_{as} \triangleq \underset{l}{\max} \{  \delta_{as}(\mbA_l) \}, 
\end{equation*}
where $\delta_{as}(\mbA_l)$ is the RIC of $\mbA_l$ and $a$ is a positive integer such as 1, 2 or 3.
\begin{theorem}[Recurrence inequality]
{\label{thm:est_iter_result_DHTP1}}
The performance of the DHTP algorithm at iteration $k$ can be bounded as 
\begin{eqnarray*}
\begin{array}{l}
\sum\limits_{l=1}^{L} \|\mbx-\hat{\mbx}_{l,k}\| \leq c_1 \sum\limits_{l=1}^{L} w_l\|\mbx-\hat{\mbx}_{l,k-1}\| +  d_{1}\sum\limits_{l=1}^{L}  w_l\|\mbe_l\|,
\end{array}
\end{eqnarray*}
where $w_l = \sum_r h_{rl}, c_1 = \sqrt{\frac{8\delta_{3s}^2}{1-\delta_{2s}^2}}, d_1 = \frac{2\sqrt{2(1-\delta_{2s})}+2\sqrt{1+\delta_{s}}}{1-\delta_{2s}}$. 
\QEDA
\end{theorem}
Detailed proof of the above theorem is shown in Section~\ref{sec:DHTP_proofs}. We use some intermediate steps in the proof of theorem~\ref{thm:est_iter_result_DHTP1} for addressing convergence of DHTP.  We show convergence by two alternative approaches, in the following two theorems.
\begin{theorem}[Convergence]
{\label{thm:num_itr_DHTP1}}
Let $x_j^{*}$ denote the magnitude of the $j$'th highest amplitude element of $\mbx$ and $\|\mbe\|_{\max} \triangleq \underset{l}{\max} \|\mbe_l\|$. If $\delta_{3s} < 1/3 $ and $\|\mbe\|_{\text{max}} \leq \gamma x_s^{*}$, then the DHTP algorithm converges after $\bar{k} = cs$ iterations, and its performance is bounded by
\begin{equation*}
\|\mbx - \hat{\mbx}_{l,\bar{k}}\| \leq d \, \|\mbe\|_{\text{max}},
\end{equation*}
where $c = \frac{\log(16c_3^2/c_1^4)}{\log(1/c_1^2)}$, $c_3 = \sqrt{\frac{16\delta_{3s}^2}{(1-\delta_{3s}^2)^2}}$,  $d = \frac{4}{\sqrt{1-\delta_{3s}}}$ and $\gamma < 1$ are positive constants in terms of $\delta_{3s}$.
Under the above conditions, estimated support sets across all nodes are equal to the correct support set, that means 
\begin{equation*}
\forall l,  \,\, \hat{\T}_l = \supp (\mathbf{x}, s).  
\end{equation*}
\QEDA
\end{theorem}
For an interpretation of the above theorem, we provide a numerical example. If $\delta_{3s} \leq 0.2$ then we have $c \leq 5$, $d \leq 4.47$; for an appropriate $\gamma$ such that $\|\mbe\|_{\text{max}} \leq \gamma x_s^{*}$, the performance $\|\mbx - \hat{\mbx}_{l,\bar{k}}\|$ is upper bounded by $4.47\|\mbe\|_{\text{max}}$ after $5s$ iterations.

\begin{corollary}
\label{cor:DHTP_bound_double_stoch}
Consider the special case of a doubly stochastic network matrix $\mbH$. Under the same conditions stated in Theorem~\ref{thm:num_itr_DHTP1}, we have 
\begin{equation*}
\|\underline{\mbx} -\underline{\hat{\mbx}_{\bar{k}}}\| \leq d \, \|\underline{\mbe}\|,
\end{equation*}
where $\underline{\mbx} =[\mbx^{\top} \hdots \mbx^{\top}]^{\top}$, $\underline{\hat{\mbx}_{\bar{k}}} =[\hat{\mbx}_{1,\bar{k}}^{\top} \hdots \hat{\mbx}_{L,\bar{k}}^{\top}]^{\top}$ and $\underline{\mbe} = [\mbe_1^{\top} \hdots \mbe_L^{\top}]^{\top}$. This upper bound is tighter than the bound of Theorem~\ref{thm:num_itr_DHTP1}.
\end{corollary}

\begin{theorem}[Convergence]
\label{thm:num_itr_DHTP2}
If $\delta_{3s} < 1/3$ and $\frac{\|\mbx\|}{\|\mbe\|_{\text{max}}} > 1$, then after $\bar{k} = \left \lceil \log \left(\frac{\|\mbx\|}{\|\mbe\|_{\text{max}}}\right) / \log\left(\frac{1}{c_1}\right) \right \rceil$ iterations, DHTP algorithm converges and its performance is bounded by
\begin{equation*}
\|\mbx - \hat{\mbx}_{l,\bar{k}}\| \leq d \, \|\mbe\|_{\text{max}},
\end{equation*}
where $d = 1+\frac{c_2 d_1}{1-c_1} + d_4$.  \QEDA
\end{theorem}
A relevant numerical example for interpretation of Theorem~\ref{thm:num_itr_DHTP2} is as follows: if $\delta_{3s} \leq 0.2$ and $\frac{\|\mbx\|}{\|\mbe\|_{\text{max}}} = 20$ dB, then we have $\bar{k} = 9$ and $d = 13.68$. 
The proofs of the theorems and corollary are presented in Section~\ref{sec:DHTP_proofs}. 

It can be seen that the DHTP algorithm  has a convergence guarantee when $\delta_{3s} < 1/3$.  Note that the requirement on signal-to-noise relation $\frac{\|\mbx\|}{\|\mbe\|_{\text{max}}} > 1$ in Theorem~\ref{thm:num_itr_DHTP2} is weaker than the requirement $\|\mbe\|_{\text{max}} \leq \gamma x_s^{*}$ in Theorem~\ref{thm:num_itr_DHTP1}. 
The above results can be readily extended to the noiseless case, that means $\forall l, \,\, \mbe_l = \mathbf{0}$. For the noiseless case, DHTP provides the exact estimate of the sparse signal $\mbx$ at every node. 

\subsection{Similarities and differences between DHTP and DiHaT}
\begin{algorithm}[t]
\caption{DiHaT - at node $l$} \label{algo:DiHaT}
\emph{Input}: $\mby_l$, $\mbA_l$, $s$, $\{ h_{lr} \}$ \\
\emph{Initialization}: $k \gets 0$; $\hat{\mbx}_{l,0} \gets \mathbf{0}$, $\bar{\mby}_{l,0} \gets \mby_l$, $\bar{\mbA}_{l,0} \gets \mbA_l$
\emph{Iteration}:
\begin{algorithmic}
\State \textbf{repeat}
\State $k \gets k+1$ \hfill (Iteration counter)
\end{algorithmic}
\begin{algorithmic}[1]
\State $\bar{\mby}_{l,k} =  \sum\limits_{r \in \mathcal{N}_{l}} h_{lr} \,\ \bar{\mby}_{r,k-1}$;  
$\bar{\mbA}_{l,k} =  \sum\limits_{r \in \mathcal{N}_{l}} h_{lr} \,\ \bar{\mbA}_{r,k-1}$ \label{step:information_fusion_y_A}
\State $\tilde{\T}_{l,k} \gets \supp (\hat{\mbx}_{l,k-1} + \bar{\mbA}^{\top}_{l,k}(\bar{\mby}_{l,k} - \bar{\mbA}_{l,k}\hat{\mbx}_{l,k-1}), s)$
\State $\tilde{\mbx}_{l,k} \,\,\ \text{such that} \,\,\ \tilde{\mbx}_{\tilde{\T}_{l,k}} \gets (\bar{\mbA}_{l,k})^{\dag}_{\tilde{\T}_{l,k}} \bar{\mby}_{l,k} \ ; \,\ \tilde{\mbx}_{\tilde{\T}^{c}_{l,k}} \gets \mathbf{0}$
\State $\check{\mbx}_{l,k} =  \sum\limits_{r \in \mathcal{N}_{l}} h_{lr} \,\ \tilde{\mbx}_{r,k}$ \label{step:information_fusion_dihat}
\State $\hat{\T}_{l,k} \gets \supp (\check{\mbx}_{l,k},s)$
\State $\hat{\mbx}_{l,k} \,\,\ \text{such that} \,\,\ \hat{\mbx}_{\hat{\T}_{l,k}} \gets \check{\mbx}_{\hat{\T}_{l,k}} \ ; \,\ \hat{\mbx}_{\hat{\T}^{c}_{l,k}} \gets \mathbf{0}$
\end{algorithmic}
\begin{algorithmic}
\State \textbf{until} \emph{stopping criterion}
\end{algorithmic}
\emph{Output}: $\hat{\mbx}_{l}, \ \hat{\mathcal{T}}_{l} $
\end{algorithm}
The DiHaT algorithm of \cite{Sergios_greedy_sparsity_algorithm_distributed_learning_TSP_2015} is shown in Algorithm~\ref{algo:DiHaT}. Comparing with Algorithm~\ref{algo:DHTP}, the similarities and differences between DHTP and DiHaT are given in the list below.
\begin{enumerate}
\item DiHaT requires exchange of $\bar{\mby}_{l,k}$, $\bar{\mbA}_{l,k}$ and $\tilde{\mbx}_{l,k}$ among nodes. DHTP requires only exchange of $\tilde{\mbx}_{l,k}$.
\item DiHaT requires $\mbH$ to be a doubly stochastic matrix. This is not a requirement for the case of DHTP.
\item For theoretical convergence proof of DiHaT, an assumption is that the average noise over nodes $\frac{1}{L} \sum\limits_{l} \mbe_l \rightarrow \mathbf{0}$. On the other hand, DHTP requires a signal-to-noise-ratio term $\frac{\|\mbx\|}{\|\mbe\|_{\text{max}}} > 1$.   
\item Denoting $\bar{\mbA} \triangleq \frac{1}{L} \sum\limits_{l} \mbA_l $, DiHaT converges if $\delta_{3s}(\bar{\mbA}) < \frac{1}{3}$. On the other hand, DHTP converges if $\underset{l}{\max} \{  \delta_{3s}(\mbA_l) \} < \frac{1}{3}$.
\item DiHaT provides consensus in the sense of achieving same estimation at all nodes under certain technical conditions. DHTP does not provide consensus except in the noiseless case under certain technical conditions.
\item DiHaT requires all $\mbA_l$ to be of the same size. DHTP does not require this condition, that means dimension of $\mathbf{y}_l$ can vary across nodes. \emph{This is an important advantage in practical scenarios.}
\end{enumerate}

\section{Simulation Results}

\begin{figure*}
\begin{center}
\includegraphics[scale=0.6]{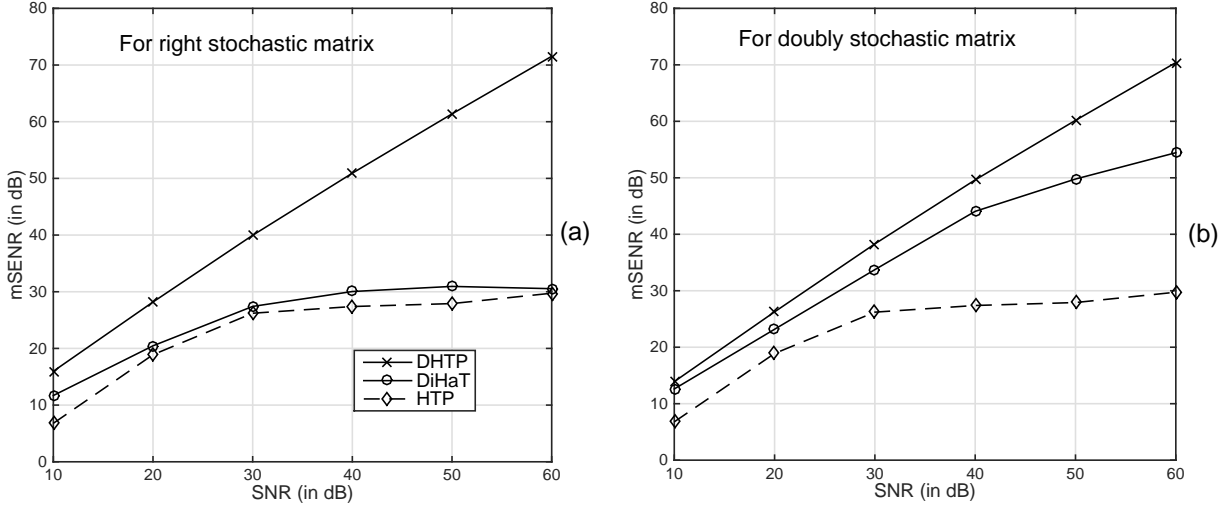}
\caption{mSENR performance of DHTP, DiHaT and HTP algorithms with respect to SNR. (a) Performance for a right stochastic network matrix. (b) Performance for a doubly stochastic network matrix.}
\label{fig:noise_robustness_plot_sm}
\end{center}
\vspace{-10mm}
\end{figure*}

\begin{figure}
\begin{center}
\includegraphics[scale=0.55]{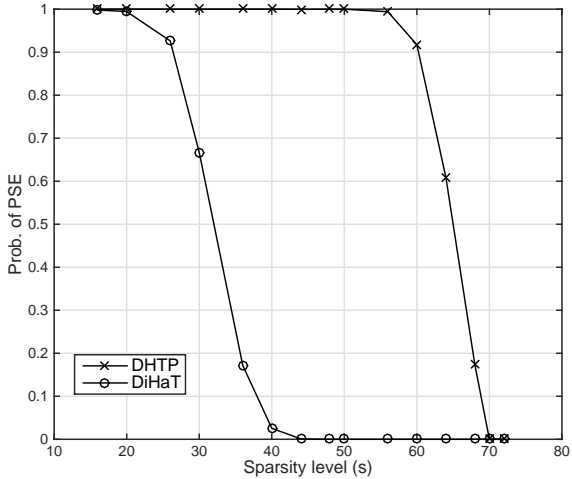}
\caption{Probability of perfect support-set estimation (PSE) versus sparsity level. No noise condition and we used a doubly stochastic network matrix.}
\label{fig:probability_of_PSE}
\end{center}
\vspace{-5mm}
\end{figure}


In this section, we study the practical performance of DHTP using simulations and compare with DiHaT  (Algorithm~\ref{algo:DiHaT}). We perform the study using Monte-Carlo simulations over many instances of $\mbA_l$, $\mbx$, and $\mbe_l$ in the system model~\eqref{eq:system_model}. The non-zero scalars of sparse signal are i.i.d. Gaussian. This is referred to as a Gaussian sparse signal. We set the number of nodes $L=20$, and every node in the network is randomly connected to three other nodes apart from itself. The stopping criterion for the algorithms is that the maximum allowable number of iterations equal to 30. We used both right stochastic and doubly stochastic $\mbH$ in simulations. Given an edge matrix of network connection between nodes, a right stochastic matrix generation is a simple task. The doubly stochastic matrix is generated through the \textit{second largest eigenvalue modulus} (SLEM) optimization problem \cite{Boyd_fastest_mixing_markov_SIAM_2004}. Finally, we also show performance for real image data.

\subsection{Performance measures}
For performance evaluation, we used a mean signal-to-estimation-noise ratio (mSENR) metric, $\text{mSENR} = \frac{1}{L} \sum_{l} \frac{\mathbb{E}{\{ \|\mbx\|^2 \}}}{\mathbb{E}{\{\|\mbx-\hat{\mbx}_l\|^2\}}}$. To generate noisy observations, we used Gaussian noise. The signal-to-noise ratio (SNR), $\text{SNR} = \text{SNR}_l = \frac{\mathbb{E}\{\|\mbx\|^2\}}{\mathbb{E}\{\|\mbe_l\|^2\}}$ is considered to be the same across all nodes.

\subsection{Experiments using Simulated Data}

We use all $\mbA_l$ that have same row size, that is, $\forall l, M_l = M$. Same row size is necessary to use DiHaT for comparison. For the experiments, we set $M=100$, and signal dimension $N=500$. In our first experiment, we compare DHTP, DiHaT and HTP for right stochastic and doubly stochastic $\mbH$. The $\mbH$ matrices are shown in the appendix.  We set sparsity level $s=20$. The results are shown in Fig.~\ref{fig:noise_robustness_plot_sm} where we show mSENR versus SNR. We recall that DiHaT was not designed for right stochastic $\mbH$ and HTP is a standalone algorithm that does not use the network. For right stochastic $\mbH$, we observe from Fig.~\ref{fig:noise_robustness_plot_sm} (a) that DiHaT does not provide considerable gain over HTP, but DHTP does. On the other hand, for doubly stochastic $\mbH$, we observe from Fig.~\ref{fig:noise_robustness_plot_sm} (b) that DiHaT provides a considerable gain over HTP, but DHTP outperforms DiHaT. The experiment validates that DHTP works for right stochastic $\mbH$. We did experiments with many instances of $\mbH$, and noted similar trend in performance. 

Next, we study the probability of perfect signal estimation under a no-noise condition. Under this condition, the probability of perfect signal estimation is equivalent to the probability of perfect support-set estimation (PSE) at all nodes. Keeping $M=100$ and $N=500$ fixed, we vary the value of $s$ and compute the probability of PSE using the frequentist approach -- how many times PSE occurred. We used the same doubly stochastic $\mbH$ of the first experiment. The result is shown in Fig.~\ref{fig:probability_of_PSE}. It can be seen that the DHTP outperforms DiHaT in the sense of phase transition from perfect to imperfect estimation.


In the third experiment, we observe convergence speed of algorithms. A fast convergence leads to less usage of communication and computational resources, and less time delay in learning. We set $s=20$. The results are shown in Fig.~\ref{fig:iter_num_plot} where we show mSENR versus the number of iterations, for the noiseless condition and 30 dB SNR. We note that the DHTP has a significantly quicker convergence. In our experiments, the DHTP achieved convergence typically within five iterations.

\begin{figure}
\begin{center}
\includegraphics[scale=0.5]{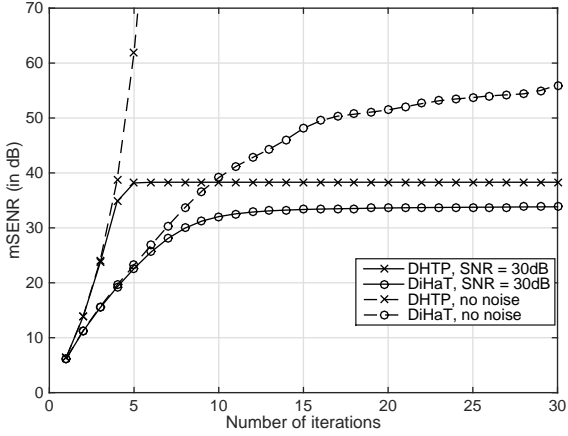}
\caption{mSENR performance of DHTP and DiHaT with respect to number of information exchange (number of iterations).}
\label{fig:iter_num_plot}
\end{center}
\end{figure}

\begin{figure}
\begin{center}
\hspace{-9mm}
\includegraphics[scale=0.5]{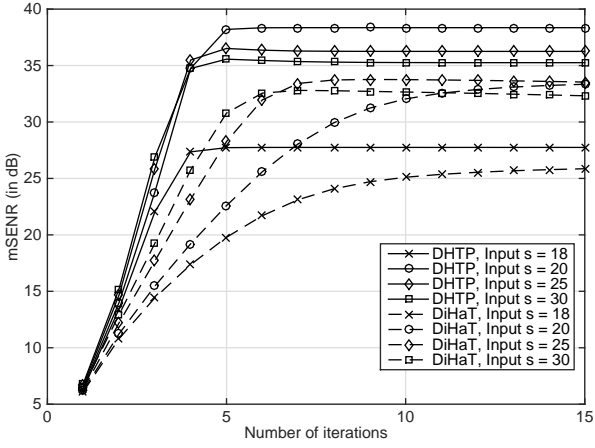}
\caption{Sensitivity performance of DHTP and DiHaT with respect to knowledge of the sparsity level. M = 100, N = 500, s = 20, L = 20, d = 4, SNR = 30dB.}
\label{fig:s_sensitivity_plot}
\end{center}
\vspace{-5mm}
\end{figure}

Finally, we experiment to find the sensitivity of the DHTP and the DiHaT algorithms to the prior knowledge of sparsity level. For this, we use 30 dB SNR and $s=20$. Fig.~\ref{fig:s_sensitivity_plot} shows the results for different assumed $s$ that varies as $s=18, 20, 25$, and $30$. We observe that DHTP performs better than DiHaT for all assumed sparsity levels. Also, a typical trend is that the assumption of higher sparsity level is always better than lower sparsity level.

\subsection{Experiments for real data}
\begin{figure*}
\begin{center}
\begin{tabular}{ccc}
\hspace{-8mm} \includegraphics[scale=0.25]{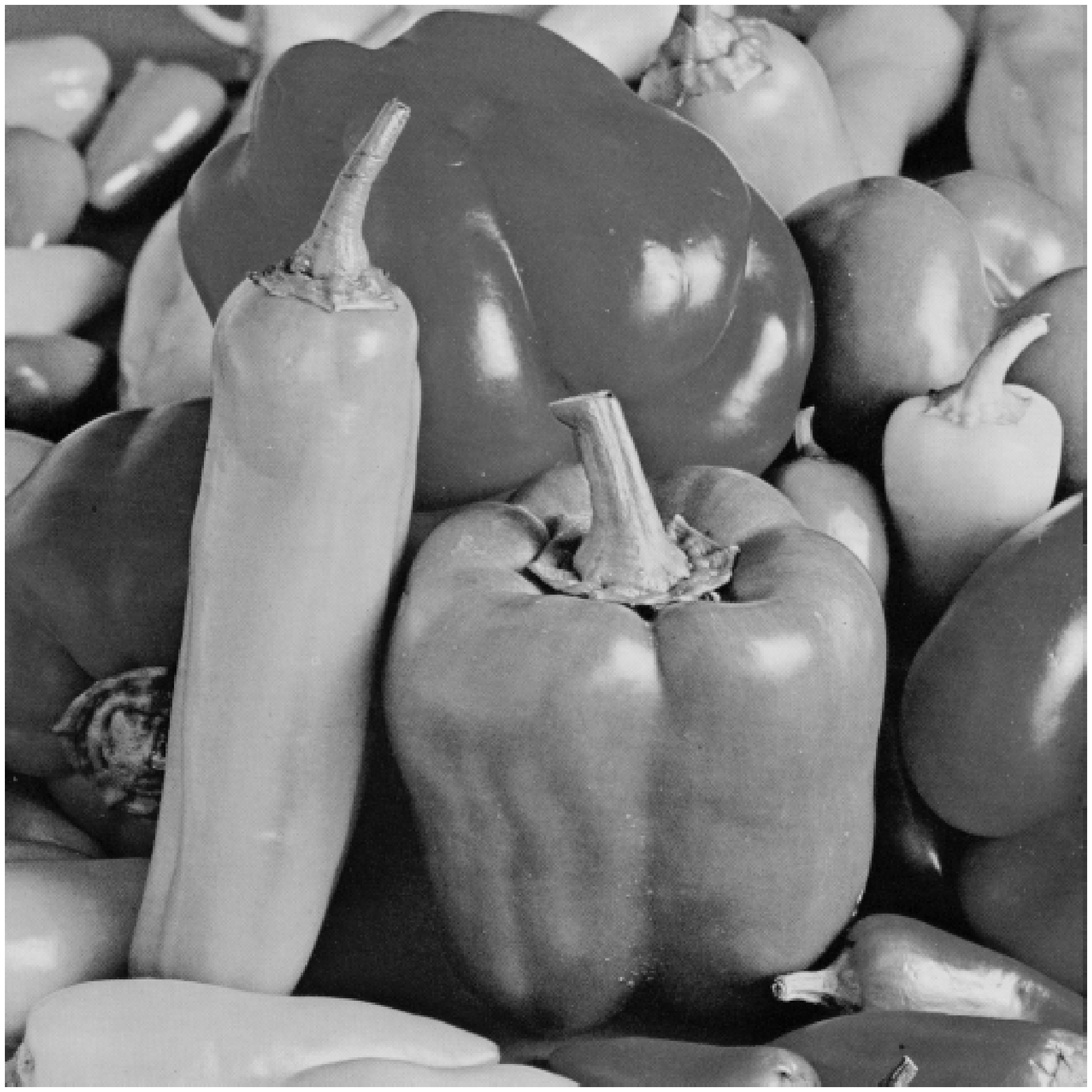} & \hspace{-15mm} \includegraphics[scale=0.25]{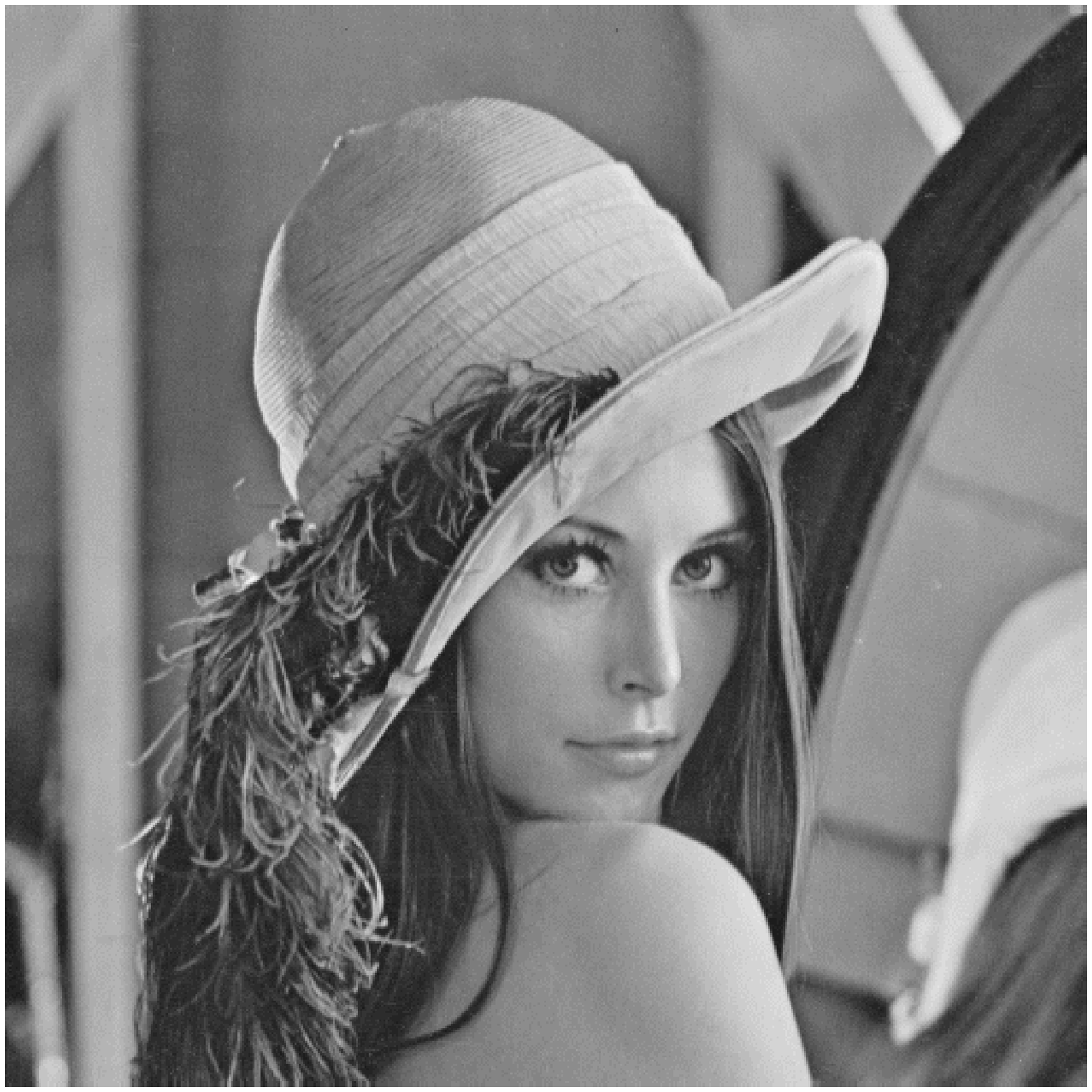} & \hspace{-15mm} \includegraphics[scale=0.25]{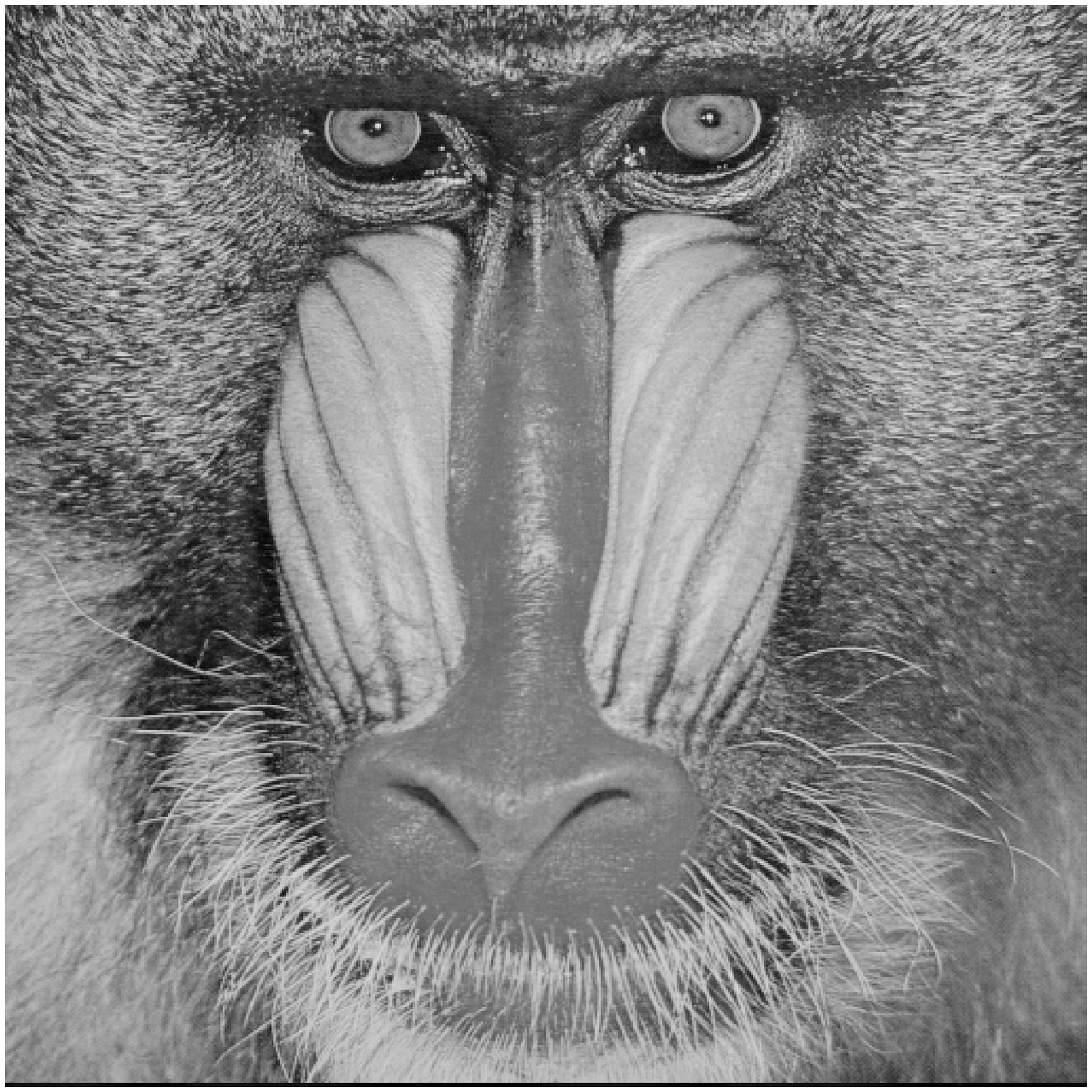} \\
\hspace{-8mm} \includegraphics[scale=0.25]{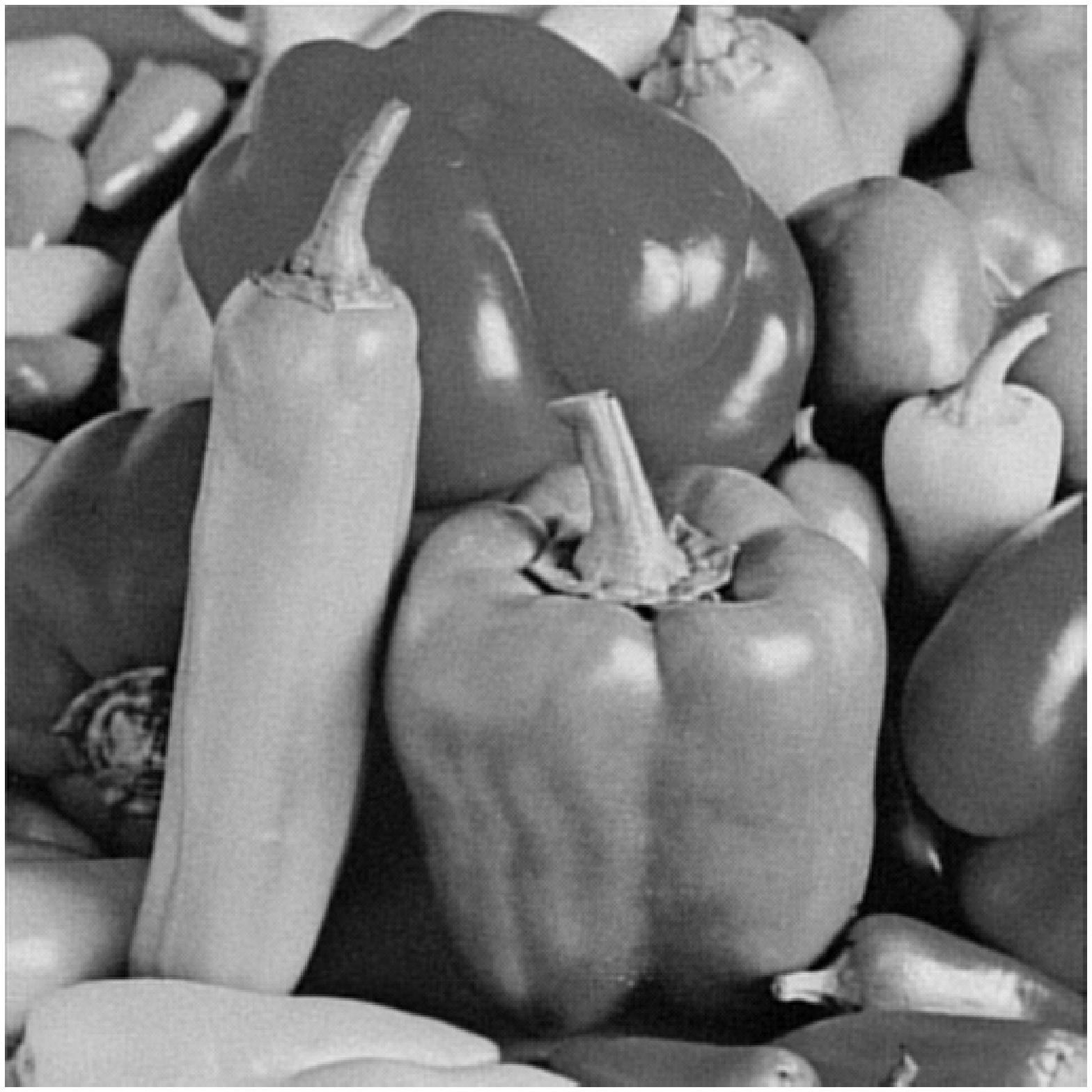} & \hspace{-15mm} \includegraphics[scale=0.25]{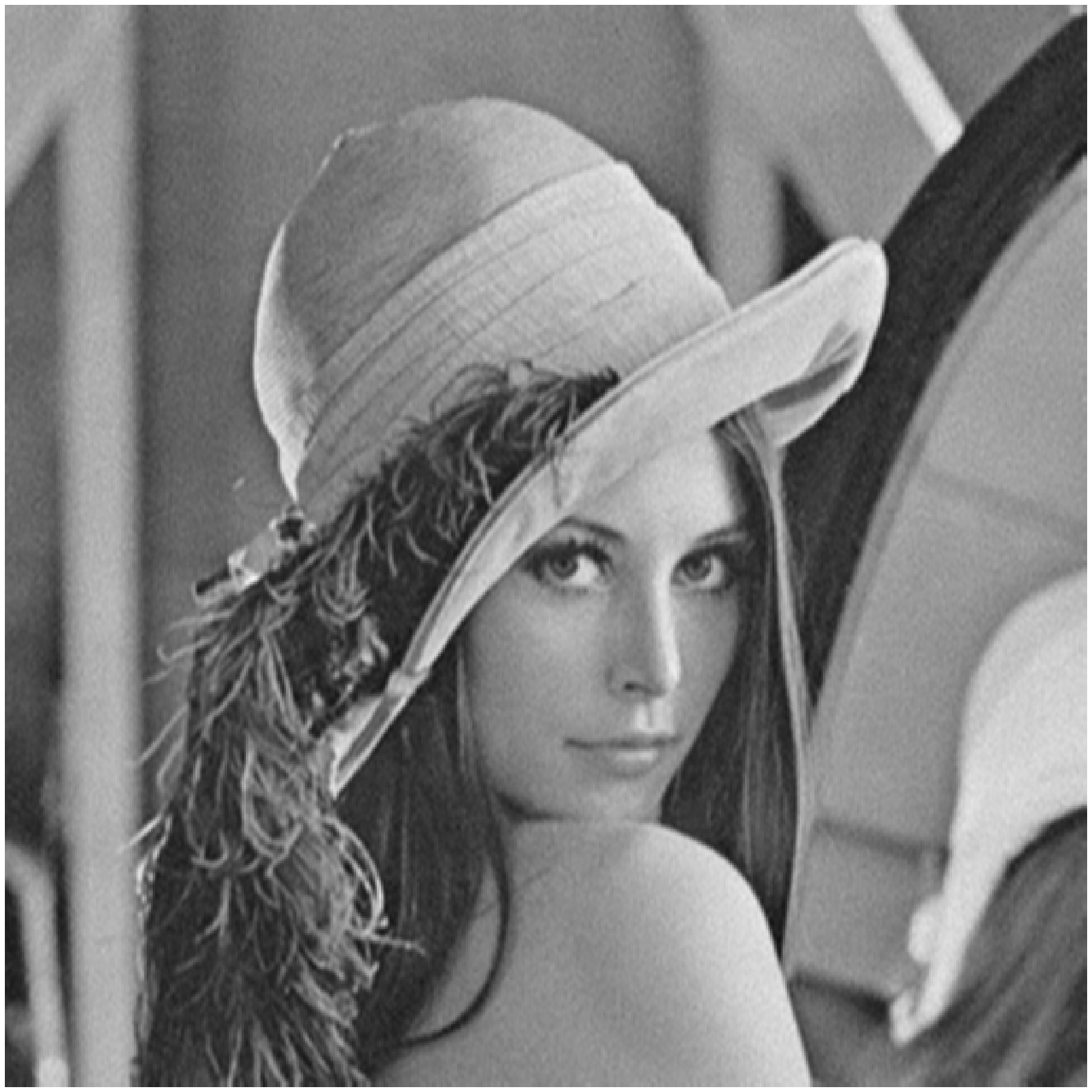} & \hspace{-15mm} \includegraphics[scale=0.25]{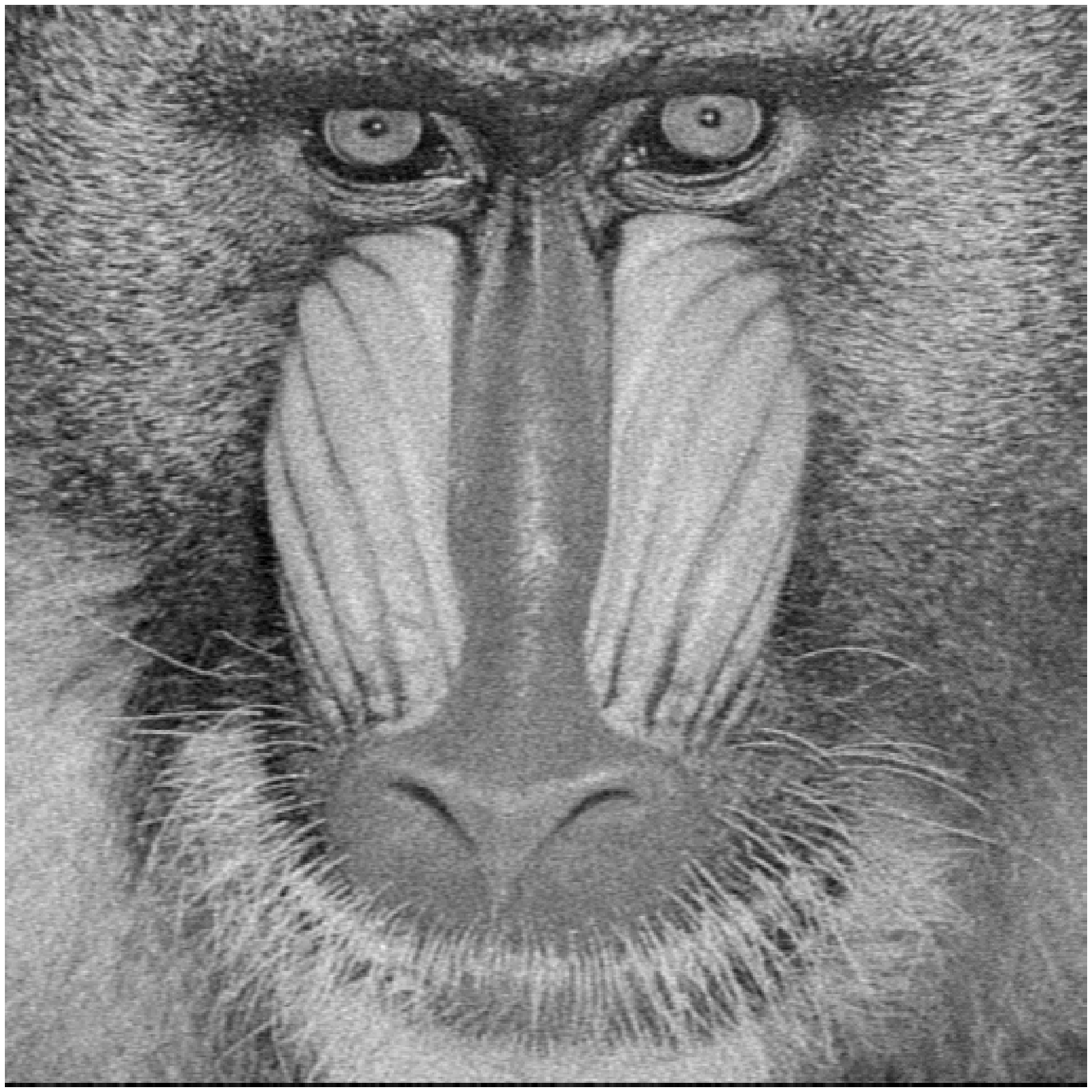} \\
\hspace{-12mm} \includegraphics[scale=0.32]{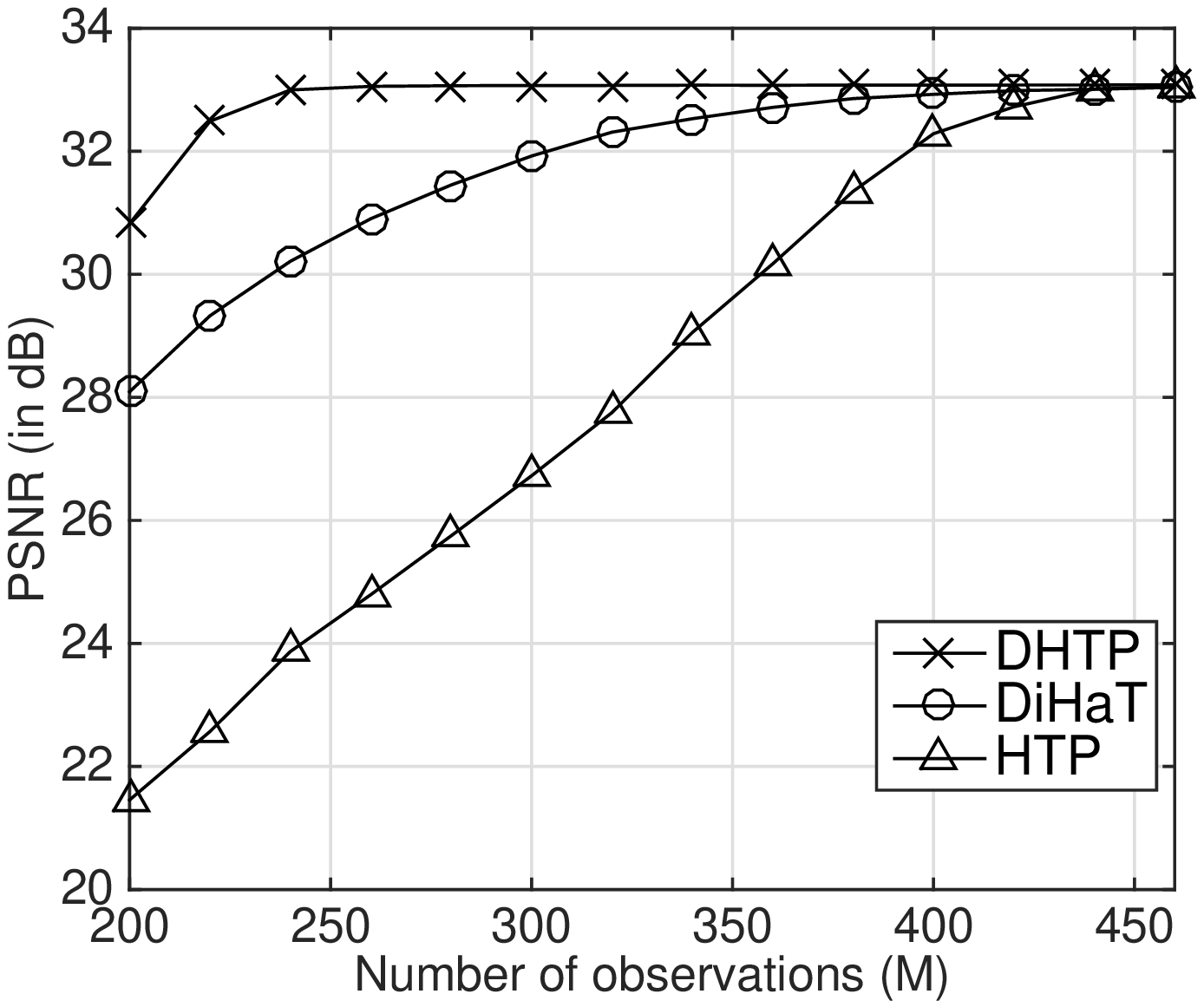} & \hspace{-18mm} \includegraphics[scale=0.32]{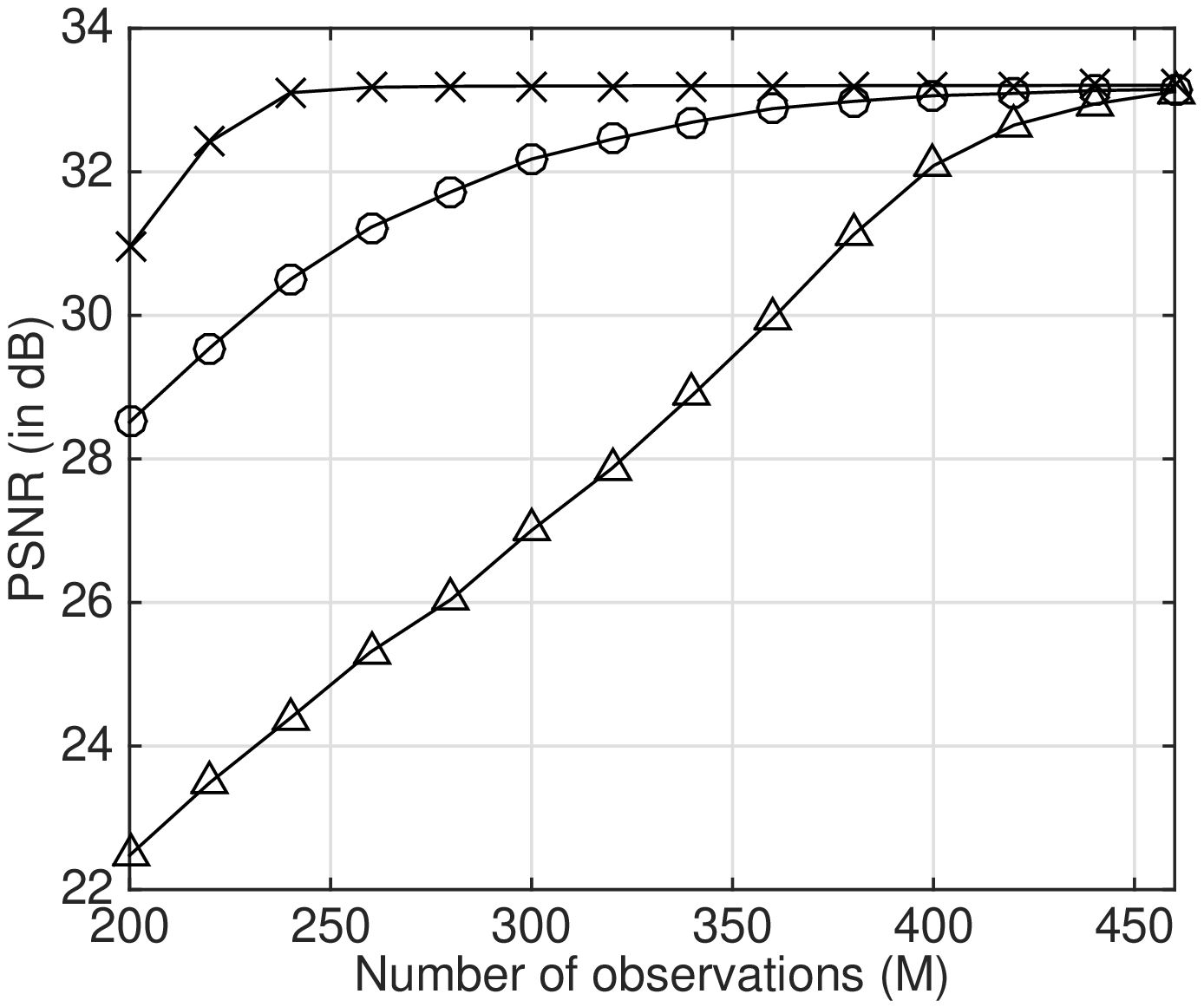} & \hspace{-15mm} \includegraphics[scale=0.32]{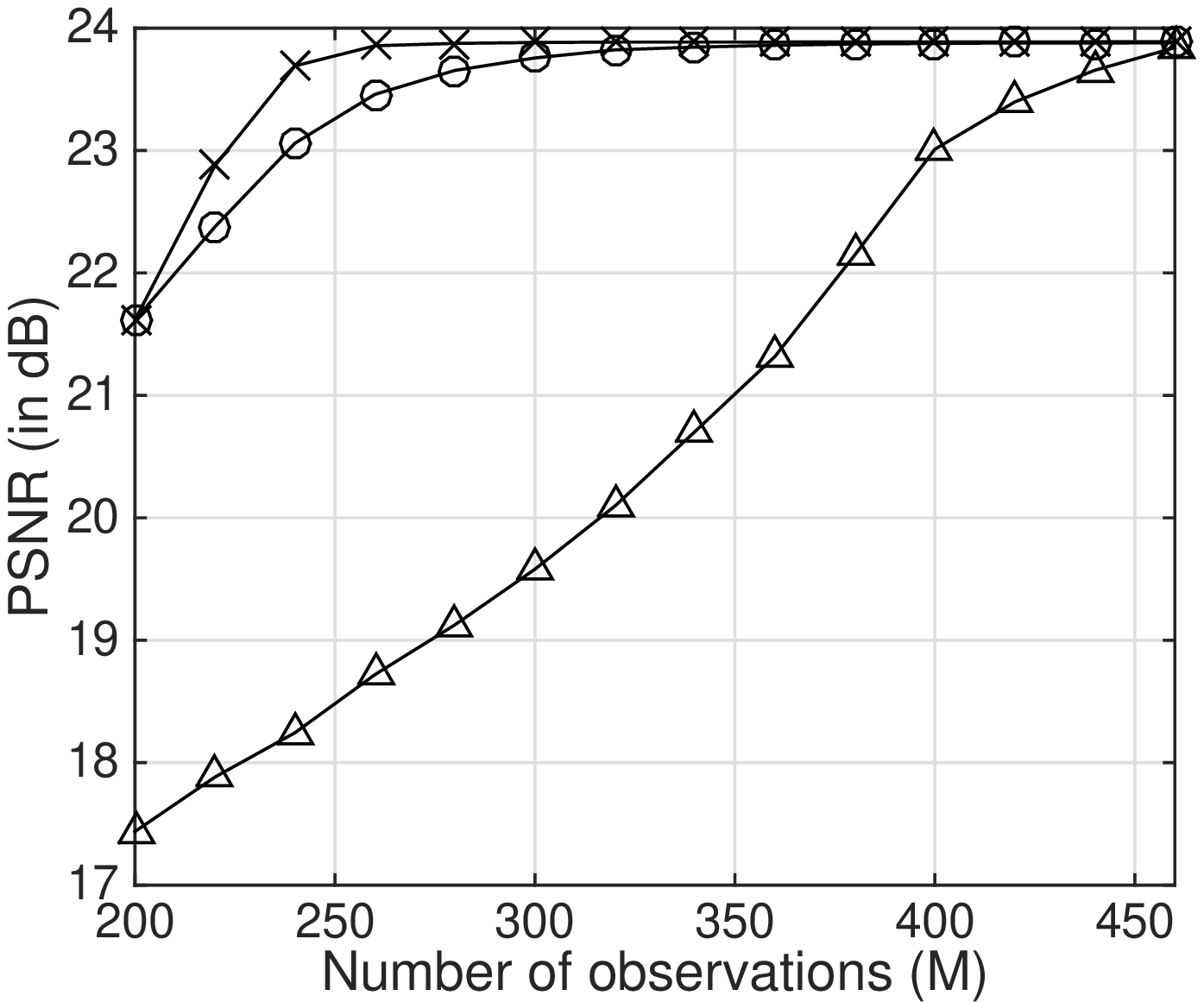}
\end{tabular}
\caption{Performance comparison of DHTP, DiHaT and HTP over image data. The first column, second column and the third column of the figure correspond to DHTP, DiHaT and HTP, respectively. The top row of the figure contains the original images. The second row contains the reconstructed images using DHTP with sparsity level $11\%$. The last row shows the PSNR performance of algorithms with respect to varying number of observations ($M$).}
\label{fig:image_multiple_fig_dhtp}
\end{center}
\end{figure*}

We evaluate the performance on three standard grayscale images: Peppers, Lena and Baboon of size $512\times512$ pixels. We consider $11\%$ of highest magnitude DCT coefficients of an image to decide a sparsity level choice. In DCT domain, the signal is split into $256$ equal parts (or blocks) for ease of computation. This leads to the value of $s$ for each part as close to $120$. We perform reconstruction of the original images using the DHTP, DiHaT and HTP algorithms over the doubly stochastic network matrix $\mbH$ chosen in the previous subsection. The performance measure is the peak-signal-to-noise-ratio (PSNR), defined as $\text{PSNR} =  \frac{\|\mbx\|_{\infty}^2}{\|\mbx-\hat{\mbx}\|^2}$,
where $\|.\|_{\infty}$ denotes the $\ell_{\infty}$ norm. We show performance for a randomly chosen node among the set of 20 nodes. Fig. \ref{fig:image_multiple_fig_dhtp} shows a plot of the PSNR versus number of observations at each node ($M$). In the same figure, we also show visual reconstruction quality (reconstructed image) at $M=240$ for DHTP. We observe that DHTP has a better convergence rate and PSNR performance than the other two algorithms. 

\subsection{Reproducible research}
In the spirit of reproducible research, we provide relevant Matlab codes at www.ee.kth.se/reproducible/ and the link https://sites.google.com/site/saikatchatt/softwares. The code produces the results shown in the figures.

\section{Conclusion}
\label{sec:conclusion}
For sparse learning over a network using distributed greedy algorithms such as the hard thresholding approach, we show that the strategy of exchanging signal estimates between nodes is good for learning. This has an explicit advantage of low communication overhead. We show that appropriate algorithmic strategies work for right stochastic network matrices. Use of right stochastic network matrices has higher generality than the popularly used doubly stochastic network matrices.

\section{Details of Theoretical Proofs}
\label{sec:DHTP_proofs}

\subsection{Useful Lemmas}
We provide three lemmas here that will be used in the proofs later. The first lemma provides a bound for the orthogonal projection operation used in DHTP.
\begin{lemma}{\cite[Lemma 2]{Liu_improved_SP_analysis_SPL_2014}}
\label{lem:zf_bound_inequality}
Consider the standard sparse representation model $\mby = \mbA\mbx + \mbe$ with $\|\mbx\|_{0} = s_1$. Let $\mathcal{S} \subseteq \{1,2,\hdots,N\}$ and $|\mathcal{S}| = s_2$. Define $\bar{\mbx}$ such that
$\bar{\mbx}_{\mathcal{S}} \gets \mbA^{\dag}_{\mathcal{S}} \mby \ ; \,\ \bar{\mbx}_{\mathcal{S}^{c}} \gets \mathbf{0}$. If $\mbA$ has RIC $\delta_{s_1+s_2} < 1$, then we have
\begin{eqnarray*}
\|\mbx-\bar{\mbx}\| \leq \sqrt{\frac{1}{1-\delta_{s_1+s_2}^2}}\|\mbx_{\mathcal{S}^{c}}\| +\frac{ \sqrt{1+\delta_{s_2}}}{1-\delta_{s_!+s_2}} \|\mbe\|.
\end{eqnarray*}
\end{lemma}
The next lemma gives a useful inequality on squares of polynomials commonly encountered in the proofs.
\begin{lemma}{\cite[Lemma 1]{Liu_improved_SP_analysis_SPL_2014}}
\label{lem:squares_bound_inequality}
For non-negative numbers $a,b,c,d,x,y,$
\begin{eqnarray*}
(ax+by)^2 + (cx+dy)^2 \leq \left( \sqrt{a^2+c^2}x + (b+d)y\right)^2.
\end{eqnarray*}
\end{lemma}
The last lemma provides a bound for the energy content in the pruned indices.
\begin{lemma}{\cite[Lemma 3]{Zaki_Venkitaraman_Chatterjee_Rasmussen_GreedySparseLearningOverNetwork_TSIPN_2017}}{\label{lem:smaller_indices_bound}}
Consider two vectors $\mbx$ and ${\mathbf{z}}$ with $\|\mbx\|_{0} = s_1$, $\|{\mathbf{z}}\|_{0} = s_2$ and $s_2 \geq s_1$. We have $\mathcal{S}_1 \triangleq \supp(\mbx, s_1)$ and $\mathcal{S}_2 \triangleq \supp({\mathbf{z}},s_2)$. Let $\mathcal{S}_{\nabla}$ denote the set of indices of the $s_2-s_1$ smallest magnitude elements in ${\mathbf{z}}$. Then,
\begin{eqnarray*}
\|\mbx_{\mathcal{S}_{\nabla}}\|  \leq \sqrt{2} \|(\mbx-{\mathbf{z}})_{\mathcal{S}_2}\| \leq \sqrt{2} \|\mbx-{\mathbf{z}}\|.
\end{eqnarray*}
\end{lemma}

\subsection{Proof of Theorem~\ref{thm:est_iter_result_DHTP1}}
At Step 2, using Lemma~\ref{lem:zf_bound_inequality}, we have
\begin{align}
\label{eq:step2_zf_bound1}
\|\mbx-\tilde{\mbx}_{l,k}\| & \leq \sqrt{\frac{1}{1-\delta_{2s}^2}}\|\mbx_{\tilde{\T}_{l,k}^c}\| + \frac{\sqrt{1+\delta_{s}}}{1-\delta_{2s}}\|\mbe_l\|, \nonumber \\
& \hspace{-5mm} \stackrel{(a)}{=} \sqrt{\frac{1}{1-\delta_{2s}^2}}\|(\mbx-\tilde{\mbx}_{l,k})_{\tilde{\T}_{l,k}^c}\| + \frac{\sqrt{1+\delta_{s}}}{1-\delta_{2s}}\|\mbe_l\|,
\end{align}
where $(a)$ follows from the construction of $\tilde{\mbx}_{l,k}$. Following the proof of \cite[Theorem 3.8]{Foucart_HTP_algorithm_SIAM_2011}, we can write,
\begin{eqnarray}
\label{eq:step1_htp_bound}
\begin{array}{l}
\hspace{-2mm} \|(\mbx-\tilde{\mbx}_{l,k})_{\tilde{\T}_{l,k}^c}\| \leq \sqrt{2}\delta_{3s}\|\mbx-\hat{\mbx}_{l,k-1}\| + \sqrt{2(1+\delta_{2s})}\|\mbe_l\|. \hspace{-4mm}
\end{array}
\end{eqnarray}
Substituting the above equation in \eqref{eq:step2_zf_bound1} , we have
\begin{eqnarray}
\label{eq:DHTP_eq_1}
\begin{array}{l}
\|\mbx-\tilde{\mbx}_{l,k}\|  \leq \sqrt{\frac{2 \delta_{3s}^2}{1-\delta_{2s}^2}} \|\mbx-\hat{\mbx}_{l,k-1}\| + \frac{d_1}{2}\|\mbe_l\|.
\end{array}
\end{eqnarray}
where $d_1 = \frac{2\sqrt{2(1-\delta_{2s})} + 2\sqrt{(1+\delta_{s})}}{1-\delta_{2s}}$.
Next, in step 3, we get
\begin{eqnarray}
\label{eq:DHTP_eq_2}
\begin{array}{rl}
\|\mbx-\check{\mbx}_{l,k}\| & \stackrel{(a)}{=} \|\sum_{r \in \mathcal{N}_l} h_{lr} \mbx - \sum_{r \in \mathcal{N}_l} h_{lr} \tilde{\mbx}_{r,k}\| \\
& \stackrel{(b)}{\leq} \sum_{r \in \mathcal{N}_l} h_{lr} \|\mbx-\tilde{\mbx}_{r,k}\|,
\end{array}
\end{eqnarray}
where $(a)$ follows as $\sum_{r \in \mathcal{N}_l} h_{lr} = 1$ and $(b)$ follows from the fact that $h_{lr}$ is non-negative. Now, we bound the performance over the pruning step in steps 4-5 as follows,
\begin{eqnarray}
\label{eq:DHTP_eq_3}
\begin{array}{rl}
\|\mbx - \hat{\mbx}_{l,k}\| & = \|(\mbx - \check{\mbx}_{l,k}) + (\check{\mbx}_{l,k} - \hat{\mbx}_{l,k})\| \\
& \hspace{-5mm} \stackrel{(a)}{\leq} \|\mbx - \check{\mbx}_{l,k}\| + \|\check{\mbx}_{l,k} - \hat{\mbx}_{l,k}\| \stackrel{(b)}{\leq} 2 \|\mbx - \check{\mbx}_{l,k}\|,
\end{array}
\end{eqnarray}
where $(a)$ follows from the triangle inequality and $(b)$ follows from the fact that $\hat{\mbx}_{l,k}$ is the best $s$-size approximation to $\check{\mbx}_{l,k}$. Combining \eqref{eq:DHTP_eq_1}, \eqref{eq:DHTP_eq_2} and \eqref{eq:DHTP_eq_3}, we get
\begin{eqnarray*}
\begin{array}{l}
\|\mbx-\hat{\mbx}_{l,k}\| \leq c_1 \sum_{r \in \mathcal{N}_l} h_{lr} \|\mbx-\hat{\mbx}_{r,k-1}\| +  d_{1}\sum_{r \in \mathcal{N}_l}  h_{lr}\|\mbe_r\|.
\end{array}
\end{eqnarray*}
Summing the above equation $\forall l$ and denoting $w_l = \sum_r h_{rl}$, we get the result of Theorem~\ref{thm:est_iter_result_DHTP1}.

\subsection{Proof of Theorem~\ref{thm:num_itr_DHTP1}}
Let $\pi$ be the permutation of indices of $\mbx$ such that $|x_{\pi(j)}| = x_j^{*}$ where $x_i^{*} \geq x_j^{*}$ for $i \leq  j$. In other words, $\mbx^{*}$ is sorted $\mbx$ in the descending order of magnitude. Assuming $\pi(\{1,2,\hdots,p\})\subseteq \hat{\T}_{l,k_1}$, we need to find the condition such that $\pi(\{1,2,\hdots,p+q\})\subseteq \hat{\T}_{l,k_2}$ where $k_2 > k_1$. \\
First, we have the following corollary.
\begin{corollary}
\label{cor:x_tilde_bound}
\begin{eqnarray*}
\begin{array}{l}
\underline{\|\mbx_{\tilde{\T}_{k}^c}\|}  \leq c_1 \mbH \underline{\|\mbx_{\tilde{\T}_{k-1}^c}\|} + \left(d_2 \mbH +d_3 \mathbf{I}\right)\underline{\|\mbe\|}, 
\end{array}
\end{eqnarray*}
where $d_2 = \frac{2\delta_{3s}\sqrt{2(1+\delta_{s})}}{1-\delta_{2s}}$, $d_3 = \sqrt{2(1+\delta_{2s})}$, $\underline{\|\mbx_{\tilde{\T}_{k}^c}\|} =[\|\mbx_{\tilde{\T}_{1,k}^c}\| \hdots \|\mbx_{\tilde{\T}_{L,k}^c}\|]^{\top}$ and $\underline{\|\mbe\|} = [\|\mbe_1\| \hdots \|\mbe_L\|]^{\top}$.
\end{corollary}
\begin{proof}
The proof of the corollary follows from the following arguments. From \eqref{eq:step1_htp_bound}, we can write
\begin{eqnarray*}
\begin{array}{l}
 \|\mbx_{\tilde{\T}_{l,k}^c}\| \leq \sqrt{2}\delta_{3s}\|\mbx-\hat{\mbx}_{l,k-1}\| + \sqrt{2(1+\delta_{2s})}\|\mbe_l\|,
\end{array}
\end{eqnarray*}
due to the construction of $\tilde{\mbx}_{l,k}$. Substituting \eqref{eq:step2_zf_bound1}, \eqref{eq:DHTP_eq_2} and \eqref{eq:DHTP_eq_3} with '$k-1$' in the above equation, we get
\begin{eqnarray*}
\begin{array}{l}
\|\mbx_{\tilde{\T}_{l,k}^c}\|  \leq c_1 \sum\limits_{r\in \mathcal{N}_l} h_{lr}\|\mbx_{\tilde{\T}_{r,k-1}^c}\| + d_2 \sum\limits_{r\in \mathcal{N}_l} h_{lr}\|\mbe_r\| + d_3 \|\mbe_l\|.
\end{array}
\end{eqnarray*}
The result follows from vectorizing the above equation.
\end{proof}
Next, Lemma \ref{lem:existence_idx_result_DHTP} derives the condition that $\pi(\{1,2,\hdots,p+q\}) \subseteq \tilde{\T}_{l,k}$, i.e., the desired indices are selected in Step 2 of DHTP. Note that we define $\underline{\|\mbe\|}_{\text{max}} = [\|\mbe\|_{\text{max}} \hdots \|\mbe\|_{\text{max}}]^{\top}$.
\begin{lemma}{\label{lem:existence_idx_result_DHTP}}
If $\forall l, \pi(\left\{1,\ 2, \hdots,\ p \right\}) \subset \tilde{\T}_{l,k_1}$ and 
\begin{eqnarray*}
\begin{array}{rl}
\mbx_{p+q}^{*} & > c_1^{k_2-k_1} \|\mbx_{\{p+1,\hdots,s\}}^{*}\|  + \frac{d_{2}+d_{3}}{1-c_1} \|\mbe\|_{\max}.
\end{array}
\end{eqnarray*}
then, $\tilde{\T}_{l,k_2} \forall l$ contains $\pi(\left\{1,\ 2, \hdots,\ p+q \right\})$.
\end{lemma}
\begin{proof}
The first part of the proof of this lemma is similar to the proof of \cite[Lemma 3]{Foucart_IHT_iterations_2016}. It is enough to prove that the $s$ highest magnitude indices of $\left(\hat{\mbx}_{l,k-1} + \mbA^{\top}_{l}(\mby_l - \mbA_l\hat{\mbx}_{l,k-1})\right)$ contains the indices $\pi(j)$ for $j \in \{1,2,\hdots,p+q\}$. Mathematically, we need to prove,
\begin{eqnarray}
\label{eq:iter2_bound_1}
\begin{array}{l}
\hspace{-5mm} \underset{j \in \{1,2,\hdots,p+q\}}{\min} \|\left(\hat{\mbx}_{l,k-1} + \mbA^{\top}_{l}(\mby_l - \mbA_l\hat{\mbx}_{l,k-1})\right)_{\pi(j)}\|  \\
\hspace{5mm} > \underset{d \in {{\T}^c}}{\max} \,\ \|\left(\hat{\mbx}_{l,k-1} + \mbA^{\top}_{l}(\mby_l - \mbA_l\hat{\mbx}_{l,k-1})\right)_{d}\|, \forall l. \hspace{-2mm}
\end{array}
\end{eqnarray}
The LHS of \eqref{eq:iter2_bound_1} can be written as
\begin{eqnarray*}
\begin{array}{l}
\hspace{-1mm}|\left(\hat{\mbx}_{l,k-1} + \mbA^{\top}_{l}(\mby_l - \mbA_l\hat{\mbx}_{l,k-1})\right)_{\pi(j)}| \\ \hspace{3mm} \stackrel{(a)}{\geq} |\mbx_{\pi(j)}| - |\left(-\mbx+\hat{\mbx}_{l,k-1} + \mbA^{\top}_{l}(\mby_l - \mbA_l\hat{\mbx}_{l,k-1})\right)_{\pi(j)}| \\
 \hspace{3mm} \geq \mbx_{p+q}^{*} - |\left(\left(\mbA_{l}^{\top}\mbA_{l}-\mathbf{I}\right) \left(\mbx-\hat{\mbx}_{l,k-1} \right)+\mbA_{l}^{\top}\mbe_l \right)_{\pi(j)}|,
\end{array}
\end{eqnarray*}
where $(a)$ follows from the reverse triangle inequality. Similarly, the RHS of \eqref{eq:iter2_bound_1} can be written as
\begin{eqnarray*}
\begin{array}{l}
\hspace{-1mm}|\left(\hat{\mbx}_{l,k-1} + \mbA^{\top}_{l}(\mby_l - \mbA_l\hat{\mbx}_{l,k-1})\right)_{d}| \\ \hspace{3mm} = |\mbx_d + \left(-\mbx+\hat{\mbx}_{l,k-1} + \mbA^{\top}_{l}(\mby_l - \mbA_l\hat{\mbx}_{l,k-1})\right)_{d}| \\
 \hspace{3mm} = |\left(\left(\mbA_{l}^{\top}\mbA_{l}-\mathbf{I}\right) \left(\mbx-\hat{\mbx}_{l,k-1} \right)+\mbA_{l}^{\top}\mbe_l \right)_{d}|.
\end{array}
\end{eqnarray*}
Using the bounds on LHS and RHS, \eqref{eq:iter2_bound_1} simplifies to
\begin{eqnarray*}
\label{eq:iter2_bound_2}
\begin{array}{l}
\mbx_{p+q}^{*} >  |\left(\left(\mbA_{l}^{\top}\mbA_{l}-\mathbf{I}\right) \left(\mbx-\hat{\mbx}_{l,k-1} \right)+\mbA_{l}^{\top}\mbe_l \right)_{\pi(j)}| \\ \hspace{15mm} + |\left(\left(\mbA_{l}^{\top}\mbA_{l}-\mathbf{I}\right) \left(\mbx-\hat{\mbx}_{l,k-1} \right)+\mbA_{l}^{\top}\mbe_l \right)_{d}|.
\end{array}
\end{eqnarray*}
Let RHS of the sufficient condition at node $l$ be denoted as $\text{RHS}_l$. Then,
\begin{eqnarray*}
\begin{array}{rl}
\text{RHS}_l \leq & \sqrt{2} |\left(\left(\mbA_{l}^{\top}\mbA_{l}-\mathbf{I}\right) \left(\mbx-\hat{\mbx}_{l,k-1} \right)+\mbA_{l}^{\top}\mbe_l \right)_{\{\pi(j), d\}}| \\
\leq & \sqrt{2}\|\left(\left(\mbA_{l}^{\top}\mbA_{l}-\mathbf{I}\right) \left( \mbx-\hat{\mbx}_{l,k-1} \right)\right)_{\{\pi(j), d\}}\| \\ & \hspace{3mm} + \sqrt{2}\|\left(\mbA_{l}^{\top}\mbe_l \right)_{\{\pi(j), d\}}\| \\
\stackrel{(a)}{\leq} & \sqrt{2}\delta_{3s}\|\mbx-\hat{\mbx}_{l,k-1}\| + \sqrt{2\left(1+\delta_{2s}\right)} \|\mbe_l\| \\
\stackrel{(b)}{\leq} & c_1\sum\limits_{r\in \mathcal{N}_l} h_{lr}\|\mbx_{\tilde{\T}_{r,k-1}^c}\| + d_2 \sum\limits_{r\in \mathcal{N}_l} h_{lr}\|\mbe_r\| + d_3 \|\mbe_l\|,
\end{array}
\end{eqnarray*}
where $(a)$ follows from \cite[Lemma 4-5]{Liu_improved_SP_analysis_SPL_2014} and $(b)$ follows from substituting \eqref{eq:step2_zf_bound1}, \eqref{eq:DHTP_eq_2} and \eqref{eq:DHTP_eq_3}. At iteration $k_2$, $\text{RHS}_l$ can be vectorized as
\begin{eqnarray*}
\begin{array}{l}
\hspace{-1mm} \underline{\text{RHS}} \hspace{-1mm} = \hspace{-1mm} [\text{RHS}_1 \hdots \text{RHS}_L]^{\top} \hspace{-1mm} \leq  c_1 \mbH \underline{\|\mbx_{\tilde{\T}_{k_2-1}^c}\|} + \left(d_2 \mbH +d_3 \mathbf{I}\right) \underline{\|\mbe\|}.
\end{array}
\end{eqnarray*}
Applying Corollary~\ref{cor:x_tilde_bound} repeatedly, we can write for $k_1 < k_2$,
\begin{eqnarray*}
\begin{array}{rl}
\underline{\text{RHS}} 
\leq & (c_1\mbH)^{k_2-k_1} \underline{\|\mbx_{\tilde{\T}_{k_1}^c}\|} \\
& \hspace{1mm} + \left(d_2 \mbH +d_3 \mathbf{I}_L\right) \left(\mathbf{I}_L  + \hdots + (c_1\mbH)^{k_2-k_1-1}\right)  \underline{\|\mbe\|} \\
  \stackrel{(a)}{\leq} &  c_1^{k_2-k_1} \underline{\|\mbx_{\{p+1,\hdots,s\}}^{*}\|}  + \frac{d_{2}+d_{3}}{1-c_1} \underline{\|\mbe\|}_{\max},
\end{array}
\end{eqnarray*}
where $(a)$ follows from the assumption that for any $l$, $\|\mbx_{\tilde{\T}_{l,k_1}^c}\| \leq \|\mbx_{\{p+1,\hdots,s\}}^{*}\|$ and the right stochastic property of $\mbH$. Now, it can be seen that the bound in \eqref{eq:iter2_bound_1} is satisfied when
\begin{eqnarray*}
\begin{array}{rl}
\mbx_{p+q}^{*} & > c_1^{k_2-k_1} \|\mbx_{\{p+1,\hdots,s\}}^{*}\|  + \frac{d_{2}+d_{3}}{1-c_1} \|\mbe\|_{\max} .
\end{array}
\end{eqnarray*}
\end{proof}
Next, we find the condition that $\pi(\{1,2,\hdots,p+q\}) \subseteq \hat{\T}_{l,k}$ in the following Lemma.
\begin{lemma}{\label{lem:select_idx_result_DHTP}}
If $\pi(\left\{1,\ 2, \hdots,\ p \right\}) \subseteq \tilde{\T}_{l,k_1} \forall l$,
$\pi(\left\{1,\ 2, \hdots,\ p+q \right\}) \subseteq \tilde{\T}_{l,k_2} \forall l$ and
\begin{eqnarray*}
\begin{array}{rl}
\hspace{-2mm}\mbx_{p+q}^{*}  > c_{3} c_1^{k_2-k_1-1} \|\mbx_{\{p+1,\hdots,s\}}^{*}\|  + \left(\frac{c_{3}(d_{2} + d_3)}{1-c_1}+d_{4} \right) \|\mbe\|_{\max}
\end{array}
\end{eqnarray*}
then, $\hat{\T}_{l,k_2}, \forall l$ contains $\pi(\left\{1,\ 2, \hdots,\ p+q \right\})$. The constant $d_4  = \frac{2d_2}{\sqrt{1-\delta_{3s}^2}}+\frac{d_1}{\sqrt{2}}$.
\end{lemma}
\begin{proof}
It is enough to prove that the $s$ highest magnitude indices of $\check{\mbx}_{l,k}$ contains the indices $\pi(j)$ for $j \in \{1,2,\hdots,p+q\}$. Mathematically, we need to prove,
\begin{eqnarray}
\label{eq:iter2_bound_DHTP1}
\underset{j \in \{1,2,\hdots,p+q\}}{\min} \|(\check{\mbx}_{l,k})_{\pi(j)}\| > \underset{d \in {{\T}^c}}{\max} \|(\check{\mbx}_{l,k})_{d}\|, \forall l.
\end{eqnarray}
The LHS of \eqref{eq:iter2_bound_DHTP1} can be written as
\begin{eqnarray*}
\begin{array}{l}
\|(\check{\mbx}_{l,k})_{\pi(j)}\|  \stackrel{(a)}{=} \left\|\left(\sum\limits_{r \in \mathcal{N}_l} h_{lr} \, \tilde{\mbx}_{\tilde{\T}_{r,k}}\right)_{\pi(j)}\right\| \\
 \hspace{5mm} \stackrel{(b)}{=} \left\|\sum\limits_{r \in \mathcal{N}_l} h_{lr} \, \mbx_{\pi(j)} + \sum\limits_{r \in \mathcal{N}_l} h_{lr} \, \left(\acute{\mbx}_{\tilde{\T}_{r,k}} - \mbx_{\tilde{\T}_{r,k}} \right)_{\pi(j)}\right\| \\
\hspace{5mm} \geq \left\|\sum\limits_{r \in \mathcal{N}_l} h_{lr} \, \mbx_{\pi(j)}\right\| - \left\|\sum\limits_{r \in \mathcal{N}_l} h_{lr} \, \left(\tilde{\mbx}_{\tilde{\T}_{r,k}} - \mbx_{\tilde{\T}_{r,k}} \right)_{\pi(j)}\right\| \\
\hspace{5mm} \geq \mbx_{p+q}^{*} - \left\|\sum\limits_{r \in \mathcal{N}_l} h_{lr} \, \left(\tilde{\mbx}_{\tilde{\T}_{r,k}} - \mbx_{\tilde{\T}_{r,k}} \right)_{\pi(j)}\right\|,
\end{array}
\end{eqnarray*}
where $(a)$ and $(b)$ follows from the fact that $\pi(\left\{1,2,\hdots,p+q\right\}) \subset \tilde{\T}_{l,k_2}, \forall l$. Similarly, the RHS of \eqref{eq:iter2_bound_DHTP1} can be bounded as
\begin{eqnarray*}
\begin{array}{rl}
\|(\check{\mbx}_{l,k})_{d}\| & = \left\|\left(\sum\limits_{r \in \mathcal{N}_l} h_{lr} \, \tilde{\mbx}_{\tilde{\T}_{r,k}}\right)_{d}\right\| \\
& = \,\, \left\|\sum\limits_{r \in \mathcal{N}_l} h_{lr} \, \mbx_{d} + \sum\limits_{r \in \mathcal{N}_l} h_{lr} \, \left(\tilde{\mbx}_{\tilde{\T}_{r,k}} - \mbx_{\tilde{\T}_{r,k}} \right)_{d}\right\| \\
& =  \left\|\sum\limits_{r \in \mathcal{N}_l} h_{lr} \, \left(\tilde{\mbx}_{\tilde{\T}_{r,k}} - \mbx_{\tilde{\T}_{r,k}} \right)_{d}\right\|.
\end{array}
\end{eqnarray*}
Using the above two bounds, the condition \eqref{eq:iter2_bound_DHTP1} can now be written as
\begin{eqnarray}
\label{eq:iter2_bound_main_DHTP}
\begin{array}{l}
\mbx_{p+q}^{*} >  \left\|\sum\limits_{r \in \mathcal{N}_l} h_{lr} \, \left(\tilde{\mbx}_{\tilde{\T}_{r,k}} - \mbx_{\tilde{\T}_{r,k}} \right)_{\pi(j)}\right\| \\ \hspace{15mm} + \left\|\sum\limits_{r \in \mathcal{N}_l} h_{lr} \, \left(\tilde{\mbx}_{\tilde{\T}_{r,k}} - \mbx_{\tilde{\T}_{r,k}} \right)_{d}\right\|.
\end{array}
\end{eqnarray}
Define the RHS of the required condition from \eqref{eq:iter2_bound_main_DHTP} at node $l$ as $\text{RHS}_l$. Then, we can write the sufficient condition as
\begin{eqnarray*}
\begin{array}{rl}
\text{RHS}_l \leq & \sqrt{2}\left\|\sum\limits_{r \in \mathcal{N}_l} h_{lr} \, \left(\tilde{\mbx}_{\tilde{\T}_{r,k}} - \mbx_{\tilde{\T}_{r,k}} \right)_{\{\pi(j),d\}}\right\| \\
\leq & \sqrt{2}\sum\limits_{r \in \mathcal{N}_l} h_{lr} \, \|\mbx-\tilde{\mbx}_{r,k}\|.
\end{array}
\end{eqnarray*}
From the above equation and \eqref{eq:DHTP_eq_1}, we can bound $\text{RHS}_l$ as
\begin{eqnarray*}
\begin{array}{rl}
\text{RHS}_l &\leq  \sum\limits_{r \in \mathcal{N}_l} h_{lr} \, \left(\sqrt{\frac{4\delta_{3s}^2}{1-\delta_{3s}^2}}\|\mbx-\hat{\mbx}_{r,k-1}\| + \frac{d_1}{\sqrt{2}}  \| \mbe_r \| \right) \\
&\stackrel{(a)}{\leq}  c_{3}\sum\limits_{r \in \mathcal{N}_l} h_{lr} \, \|\mbx_{\tilde{\T}_{r,k-1}^c}\| + d_{4} \sum\limits_{r \in \mathcal{N}_l} h_{lr} \, \|\mbe_{r}\|,
\end{array}
\end{eqnarray*}
where $(a)$ follows from substituting \eqref{eq:step2_zf_bound1}, \eqref{eq:DHTP_eq_2} and \eqref{eq:DHTP_eq_3}. At iteration $k_2$, $\text{RHS}_l$ can be vectorized as
\begin{eqnarray*}
\begin{array}{rl}
\underline{\text{RHS}} = \left[\text{RHS}_1 \hdots \text{RHS}_L\right]^{\top} \leq & c_{3} \mbH\underline{\|\mbx_{\tilde{\T}_{k_2-1}^c}\|} + d_4 \mbH \underline{\|\mbe\|}.
\end{array}
\end{eqnarray*}
Applying Corollary~\ref{cor:x_tilde_bound} repeatedly, we can write for $k_1<k_2$,
\begin{eqnarray*}
\begin{array}{rl}
\underline{\text{RHS}} & \leq  c_{3} (c_1\mbH)^{k_2-k_1-1} \mbH \underline{\|\mbx_{\tilde{\T}_{k_1}^c}\|} \\ & \hspace{-10mm} + \left(c_{3}(d_{2}\mbH+d_3 \mathbf{I}) \left(\mathbf{I}_L  + \hdots + (c_1\mbH)^{k_2-k_1-2}\right) + d_{4}\right)\mbH  \underline{\|\mbe\|} \\
 & \hspace{-10mm} \stackrel{(a)}{\leq}   c_{3} c_1^{k_2-k_1-1} \underline{\|\mbx_{\{p+1,\hdots,s\}}^{*}\|}  + \left(\frac{c_{3}(d_{2} + d_3)}{1-c_1}+d_{4} \right) \underline{\|\mbe\|}_{\max},
\end{array}
\end{eqnarray*}
where $(a)$ follows from the assumption that for any $l$, $\|\mbx_{\tilde{\T}_{l,k_1}^c}\| \leq \|\mbx_{\{p+1,\hdots,s\}}^{*}\|$ and the right stochastic property of $\mbH$.
From the above bound on $\underline{\text{RHS}} $, it can be easily seen that \eqref{eq:iter2_bound_main_DHTP} is satisfied when
\begin{eqnarray*}
\begin{array}{l}
\hspace{-2mm}\mbx_{p+q}^{*}  > c_{3} c_1^{k_2-k_1-1} \|\mbx_{\{p+1,\hdots,s\}}^{*}\|  + \left(\frac{c_{3}(d_{2} + d_3)}{1-c_1}+d_{4} \right) \|\mbe\|_{\max} .
\end{array}
\end{eqnarray*}
\end{proof}
\begin{corollary}
For $\delta_{3s} < 1$, we have $c_1 < c_{3}$. Therefore, the sufficient condition for Lemma \ref{lem:existence_idx_result_DHTP} and Lemma \ref{lem:select_idx_result_DHTP} to hold is 
$\mbx_{p+q}^{*}  > c_{3} c_1^{k_2-k_1-1} \|\mbx_{\{p+1,\hdots,s\}}^{*}\|  + \left(\frac{c_{3}(d_{2} + d_3)}{1-c_1}+d_{4} \right) \|\mbe\|_{\max}$.
\end{corollary}

Next, the above corollary can be used to prove theorem~\ref{thm:num_itr_DHTP1} by similar steps as outlined in \cite[Theorem 6]{Foucart_IHT_iterations_2016}. In the proof, the number of iterations $k_i$ between different steps for our case should be defined as $k_i := \left\lceil\frac{\log \left(16\left(c_3/c_1\right)^2\left(|Q_i|+|Q_{i+1}|/2+ \hdots + |Q_r|/2^{r-i}\right)\right)}{\log \left(1/c_1^2\right)}\right\rceil$.
All the variables in the above definition are defined in \cite[Theorem 6]{Foucart_IHT_iterations_2016} except $c_1, c_3$ which are defined in this paper. With the above definition, the number of iterations is bounded as $\bar{k} = cs$. The proof steps in \cite[Theorem 6]{Foucart_IHT_iterations_2016} require that $c_1 < 1$ which holds as $\delta_{3s} < 0.362$. \\
Note that the constant $\gamma$ is defined as $\gamma = \frac{2\sqrt{2}-1}{4d_{5}}$, where $d_{5} = \frac{c_{3}(d_{2} + d_3)}{1-c_1}+d_{4}$.
The constant $d$ in the theorem statement can be derived as follows.
From Step 2 of DHTP at node $l$, we have $\|\mby_l - \mbA_l\tilde{\mbx}_{l,\bar{k}}\| \leq \|\mby_l-\mbA_l \mbx\| = \|\mbe_l\| \leq \|\mbe\|_{\text{max}}$. This follows because $\tilde{\T}_{l,\bar{k}} = \T$. Next, we can write
\begin{eqnarray*}
\begin{array}{l}
\|\mbx - \hat{\mbx}_{l,\bar{k}}\| \\ \stackrel{(a)}{\leq} 2\sum\limits_{r\in\mathcal{N}_l}h_{lr}\|\mbx - \tilde{\mbx}_{r,\bar{k}}\|  \leq \sum\limits_{r\in\mathcal{N}_l}h_{lr}\frac{2}{\sqrt{1-\delta_{2s}}} \|\mbA_r \left(\mbx - \tilde{\mbx}_{r,\bar{k}}\right)\| \\
\leq \sum\limits_{r\in\mathcal{N}_l}h_{lr}\frac{2}{\sqrt{1-\delta_{2s}}} \left(\|\mby_r - \mbA_r\tilde{\mbx}_{r,\bar{k}}\|+\|\mbe_r\|\right) \leq \frac{4}{\sqrt{1-\delta_{3s}}}\|\mbe\|_{\text{max}},
\end{array}
\end{eqnarray*}
where $(a)$ follows from \eqref{eq:DHTP_eq_2}, \eqref{eq:DHTP_eq_3} and the last inequality follows from the fact that $\delta_{2s} < \delta_{3s}$ and the right stochastic property of $\mbH$. \QEDB

\subsection{Proof of Corollary~\ref{cor:DHTP_bound_double_stoch}}
From Step 2 of DHTP, we have $\|\underline{\mby} - \underline{\mbA} \underline{\tilde{\mbx}_{\bar{k}}}\|^2 \leq \|\underline{\mby} - \underline{\mbA} \ \underline{\mbx}\|^2 = \|\underline{\mbe}\|^2$, where  $\underline{\mby} =[\mby_1^{\top} \hdots \mby_L^{\top}]^{\top}$, $\underline{\tilde{\mbx}_{\bar{k}}} =[\tilde{\mbx}_{1,\bar{k}}^{\top} \hdots \tilde{\mbx}_{L,\bar{k}}^{\top}]^{\top}$ and $\underline{\mbA} = \mathbf{I}_L \otimes \mbA$. The first inequality follows because $\tilde{\T}_{l,\bar{k}} = \T, \forall l$. Define $\underline{\mbH} \triangleq \mbH \otimes \mathbf{I}_L$. Then, we have, 
\begin{eqnarray*}
\begin{array}{l}
\|\underline{\mbx} - \underline{\hat{\mbx}_{\bar{k}}}\|^2 \stackrel{(a)}{\leq} 4 \|\underline{\mbx} - \underline{\check{\mbx}_{\bar{k}}}\|^2 \stackrel{(b)}{\leq} 4 \|\underline{\mbH} \ \underline{\mbx} - \underline{\mbH} \underline{\tilde{\mbx}_{\bar{k}}}\|^2 \stackrel{(c)}{\leq} 4 \|\underline{\mbx} - \underline{\tilde{\mbx}_{\bar{k}}}\|^2 \\
\hspace{5mm} \stackrel{(d)}{\leq} \frac{4}{1-\delta_{2s}} \left\|\underline{\mbA} \left( \underline{\mbx} - \underline{\tilde{\mbx}_{\bar{k}}}\right)\right\|^2 = \frac{4}{1-\delta_{2s}}\left\| \underline{\mby} - \underline{\mbe} - \underline{\mbA} \underline{\tilde{\mbx}_{\bar{k}}}\right\|^2,
\end{array}
\end{eqnarray*}
where $(a)$ follows from \eqref{eq:DHTP_eq_3}, and $(b)$, $(c)$ follows from the doubly stochastic property of $\underline{\mbH}$ ($\|\underline{\mbH}\| = 1$). Also, $(d)$ follows from the RIP property of $\underline{\mbA}$. Further, we can write
\begin{eqnarray*}
\begin{array}{l}
\|\underline{\mbx} - \underline{\hat{\mbx}_{\bar{k}}}\| \leq \frac{2}{\sqrt{1-\delta_{2s}}} \left(\|\underline{\mby} - \underline{\mbA} \underline{\tilde{\mbx}_{\bar{k}}}\|+\|\underline{\mbe}\|\right) \leq \frac{4}{\sqrt{1-\delta_{3s}}} \|\underline{\mbe}\|,
\end{array}
\end{eqnarray*}
where the last inequality follows from the fact that $\delta_{2s} < \delta_{3s}$.

\subsection{Proof of Theorem~\ref{thm:num_itr_DHTP2}}
We define,  $\nabla \hat{\T}_{l,k} \triangleq \{\underset{r \in \mathcal{N}_l}{\cup} \tilde{\T}_{r,k}\} \setminus \hat{\T}_{l,k}$. Then, at step 5 of DHTP we have
\begin{eqnarray}
\label{eq:step5_bound1_DHTP}
\|\mbx_{\hat{\T}_{l,k}^c}\|^2  & =  \|\mbx_{\nabla \hat{\T}_{l,k}}\|^2 + \|\mbx_{ \left\{\underset{r \in \mathcal{N}_l}{\cap} \tilde{\T}_{r,k}^c \right\} }\|^2  \nonumber \\
 & \stackrel{(a)}{\leq}  \|\mbx_{\nabla \hat{\T}_{l,k}}\|^2 + \|\sum\limits_{r \in \mathcal{N}_l} h_{lr} \, {\mbx}_{\tilde{\T}_{r,k}^c}\|^2 ,
\end{eqnarray}
where $(a)$ follows from the right stochastic property of $\mbH$.
Also, from Lemma \ref{lem:smaller_indices_bound}, we have
\begin{align}
\label{eq:step7_small_idx_bound1_DHTP}
\|\mbx_{\nabla \hat{\T}_{l,k}}\| & \leq \sqrt{2}\|\mbx-\check{\mbx}_{l,k}\| 
 \stackrel{(a)}{=} \sqrt{2}\|\sum\limits_{r \in \mathcal{N}_l} h_{lr} \, \mbx - \sum\limits_{r \in \mathcal{N}_l} h_{lr} \, \tilde{\mbx}_{r,k}\| \nonumber \\
& \stackrel{(b)}{\leq} \sqrt{2}\sum\limits_{r \in \mathcal{N}_l} h_{lr} \,\|\mbx-\tilde{\mbx}_{r,k}\|,
\end{align}
where $(a)$ and $(b)$ follows from the assumption that $\forall l, \sum\limits_{r \in \mathcal{N}_l} h_{lr} = 1$ and $\forall \{l,r\}$, the value, $h_{lr} \geq 0 $ respectively.
Combining \eqref{eq:step5_bound1_DHTP} and \eqref{eq:step7_small_idx_bound1_DHTP}, we have
\begin{eqnarray*}
\begin{array}{l}
 \hspace{-2mm}\|\mbx_{\hat{\T}_{l,k}^c}\|^2 \hspace{-1mm}
 \stackrel{\eqref{eq:step1_htp_bound}, \eqref{eq:DHTP_eq_1}}{\leq} \hspace{-1mm}\left[ \sum\limits_{r \in \mathcal{N}_l} h_{lr} \, \left( \sqrt{\frac{4\delta_{3s}^2}{1-\delta_{3s}^2}}\|\mbx-\hat{\mbx}_{r,k-1}\|  +  \frac{d_{1}}{\sqrt{2}} \|\mbe_r\| \right) \right]^2 \\
  \hspace{2mm} + \left[\sum\limits_{r \in \mathcal{N}_l} h_{lr} \, \left( \sqrt{2} \delta_{3s}\|\mbx-\hat{\mbx}_{r,k-1}\|  +  \sqrt{2(1+\delta_{2s})}\| \mbe_r \| \right) \right]^2.
\end{array}
\end{eqnarray*}
Using Lemma~\ref{lem:squares_bound_inequality}, the above equation can be simplified as
\begin{eqnarray*}
\label{eq:bound_1_DHTP}
\begin{array}{l}
\hspace{-2mm}\|\mbx_{\hat{\T}_{l,k}^c}\|  \leq  c_2 \sum\limits_{r \in \mathcal{N}_l} h_{lr} \, \|\mbx-\hat{\mbx}_{r,k-1}\|  + (\frac{d_{1}}{\sqrt{2}} + d_{3})\sum\limits_{r \in \mathcal{N}_l} h_{lr} \, \|\mbe_r\|,
\end{array}
\end{eqnarray*}
where $c_2 = \sqrt{\frac{2\delta_{3s}^2(3-\delta_{3s}^2)}{1-\delta_{3s}^2}}$.
This equation can be vectorized as 
\begin{eqnarray}
\label{eq:bound_2_DHTP}
\begin{array}{l}
\underline{\|\mbx_{\hat{\T}_{k}^c}\|}  \leq  c_2 \mbH \underline{\|\mbx-\hat{\mbx}_{k-1}\|} + d_4 \underline{\|\mbe\|},
\end{array}
\end{eqnarray}
where $\underline{\|\mbx-\hat{\mbx}_{r,k-1}\|} = [\|\mbx-\hat{\mbx}_{1,k-1}\| \hdots \|\mbx-\hat{\mbx}_{L,k-1}\|]^{\top}$ and $d_4=\frac{d_{1}}{\sqrt{2}} + d_{3}$.
From the proof of Theorem~\ref{thm:est_iter_result_DHTP1}, we have the following relation in vectorized form
\begin{eqnarray*}
\begin{array}{l}
\underline{\|\mbx-\hat{\mbx}_{k}\|} \leq c_1 \mbH \underline{\|\mbx-\hat{\mbx}_{k-1}\|} +  d_{1}\mbH\underline{\|\mbe\|}.
\end{array}
\end{eqnarray*}
Applying the above relation repeatedly, and using the fact that $c_1 < 1$ (as $\delta_{3s} < 0.362$), we can write \eqref{eq:bound_2_DHTP} at the iteration $\bar{k}$ as
\begin{eqnarray*}
\begin{array}{rl}
\underline{\|\mbx_{\hat{\T}_{\bar{k}}^c}\|}  \leq & c_2 \left(c_1 \mbH \right)^{\bar{k}-1} \mbH\underline{\|\mbx-\hat{\mbx}_{0}\|} \\
& \hspace{-8mm} + \left(c_{2}d_{1} \mbH \left(\mathbf{I}_L  + \hdots + (c_1\mbH)^{\bar{k}-2}\right) + d_4 \right) \underline{\|\mbe\|}_{\max} \\
\stackrel{(a)}{\leq} & c_1^{\bar{k}} \underline{\|\mbx\|} + \left(\frac{c_2 d_1}{1-c_1} + d_4\right) \underline{\|\mbe\|}_{\max},
\end{array}
\end{eqnarray*}
where $(a)$ follows from the fact that $c_2<c_1$, the initial condition, $\|\mbx-\hat{\mbx}_{l,0}\| = \|\mbx\|, \forall l$ and the right stochastic property of $\mbH$. We have also defined, $\underline{\|\mbx\|} = [\|\mbx\| \hdots \|\mbx\|]^{\top}$. Substituting the value of $\bar{k}$ in the above equation, we get
\begin{eqnarray*}
\begin{array}{rl}
\underline{\|\mbx_{\hat{\T}_{\bar{k}}^c}\|}  \leq & \left(1+\frac{c_2 d_1}{1-c_1} + d_4\right) \underline{\|\mbe\|}_{\max}.
\end{array}
\end{eqnarray*}
\QEDB

\bibliographystyle{ieeetr}
\bibliography{references/ref_zaki,references/ref_distributed_sparse,references/biblio_saikat_CS1,references/biblio_saikat_Pub,references/biblio_saikat_BigData}


\onecolumn
\appendix

\subsection{Right stochastic matrix used in simulation experiment}

\begin{eqnarray*}
\scriptsize{
\mbH = \left[\begin{array}{cccccccccccccccccccc}
0.28 & 0 & 0 & 0 & 0 & 0 & 0.24 & 0 & 0 & 0 & 0 & 0 & 0 & 0 & 0 & 0 & 0 & 0.30 & 0 & 0.18 \\
0 & 0.25 & 0 & 0 & 0 & 0.16 & 0.24 & 0 & 0.35 & 0 & 0 & 0 & 0 & 0 & 0 & 0 & 0 & 0 & 0 & 0 \\
0 & 0 &  0.32 & 0 & 0.14 & 0 & 0 & 0 & 0 & 0 & 0.16 & 0 & 0 & 0.38 & 0 & 0 & 0 & 0 & 0 & 0 \\
0 & 0 &  0 & 0.28 & 0 & 0 & 0 & 0.22 & 0 & 0 & 0 & 0 & 0.19 & 0 & 0.31 & 0 & 0 & 0 & 0 & 0 \\
0 & 0 &  0.18 & 0 & 0.28& 0 & 0.35 & 0 & 0 & 0.19 & 0 & 0 & 0 & 0 & 0 & 0 & 0 & 0 & 0 & 0 \\
0 & 0.19 &  0 & 0 & 0 & 0.12 & 0 & 0 & 0 & 0 & 0 & 0.28 & 0 & 0.41 & 0 & 0 & 0 & 0 & 0 & 0 \\
0.24 & 0.25 &  0 & 0 & 0.28 & 0 & 0.23 & 0 & 0 & 0 & 0 & 0 & 0 & 0 & 0 & 0 & 0 & 0 & 0 & 0 \\
0 & 0 &  0 & 0.26 & 0 & 0 & 0 & 0.12 & 0.22 & 0 & 0 & 0 & 0 & 0 & 0.40 & 0 & 0 & 0 & 0 & 0 \\
0 & 0.33 &  0 & 0 & 0 & 0 & 0 & 0.24 & 0.32 & 0 & 0 & 0 & 0 & 0 & 0 & 0 & 0.11 & 0 & 0 & 0 \\
0 & 0 &  0 & 0 & 0.04 & 0 & 0 & 0 & 0 & 0.37 & 0 & 0 & 0 & 0 & 0 & 0.19 & 0 & 0 & 0.40 & 0 \\
0 & 0 &  0.27 & 0 & 0 & 0 & 0 & 0 & 0 & 0 & 0.09 & 0 & 0 & 0 & 0 & 0.56 & 0.08 & 0 & 0 & 0 \\
0 & 0 &  0 & 0 & 0 & 0.20 & 0 & 0 & 0 & 0 & 0 & 0.22 & 0.31 & 0 & 0 & 0.27 & 0 & 0 & 0 & 0 \\
0 & 0 &  0 & 0.22 & 0 & 0 & 0 & 0 & 0 & 0 & 0 & 0.26 & 0.26 & 0 & 0 & 0.26 & 0 & 0 & 0 & 0 \\
0 & 0 &  0.06 & 0 & 0 & 0.31 & 0 & 0 & 0 & 0 & 0 & 0 & 0 & 0.24 & 0 & 0 & 0 & 0 & 0 & 0.39 \\
0 & 0 &  0 & 0.20 & 0 & 0 & 0 & 0.30 & 0 & 0 & 0 & 0 & 0 & 0 & 0.08 & 0.42 & 0 & 0 & 0 & 0 \\
0 & 0 &  0 & 0 & 0 & 0 & 0 & 0 & 0 & 0.05 & 0.06 & 0.39 & 0.20 & 0 & 0.24 & 0.06 & 0 & 0 & 0 & 0 \\
0 & 0 &  0 & 0 & 0 & 0 & 0 & 0 & 0.39 & 0 & 0.24 & 0 & 0 & 0 & 0 & 0 & 0.27 & 0 & 0 & 0.10 \\
0.27 & 0 &  0 & 0 & 0 & 0 & 0 & 0 & 0 & 0 & 0 & 0 & 0 & 0 & 0 & 0 & 0 & 0.24 & 0.23 & 0.26 \\
0 & 0 &  0 & 0 & 0 & 0 & 0 & 0 & 0 & 0.29 & 0 & 0 & 0 & 0 & 0 & 0 & 0 & 0.16 & 0.23 & 0.32 \\
0.10 & 0 &  0 & 0 & 0 & 0 & 0 & 0 & 0 & 0 & 0 & 0 & 0 & 0.21 & 0 & 0 & 0.24 & 0.04 & 0.29 & 0.12 
\end{array} \right]. }
\end{eqnarray*}

\subsection{Doubly stochastic matrix used in simulation experiment}

\begin{eqnarray*}
\scriptsize{
\mbH = \left[\begin{array}{cccccccccccccccccccc}
0.57 & 0 & 0 & 0 & 0 & 0 & 0.20 & 0 & 0 & 0 & 0 & 0 & 0 & 0 & 0 & 0 & 0 & 0.12 & 0 & 0.11 \\
0 & 0.49 & 0 & 0 & 0 & 0.10 & 0.26 & 0 & 0.15 & 0 & 0 & 0 & 0 & 0 & 0 & 0 & 0 & 0 & 0 & 0 \\
0 & 0 &  0.42 & 0 & 0.26 & 0 & 0 & 0 & 0 & 0 & 0.25 & 0 & 0 & 0.07 & 0 & 0 & 0 & 0 & 0 & 0 \\
0 & 0 &  0 & 0.56 & 0 & 0 & 0 & 0.21 & 0 & 0 & 0 & 0 & 0.13 & 0 & 0.11 & 0 & 0 & 0 & 0 & 0 \\
0 & 0 &  0.26 & 0 & 0.61 & 0 & 0.06 & 0 & 0 & 0.07 & 0 & 0 & 0 & 0 & 0 & 0 & 0 & 0 & 0 & 0 \\
0 & 0.10 &  0 & 0 & 0 & 0.40 & 0 & 0 & 0 & 0 & 0 & 0.28 & 0 & 0.22 & 0 & 0 & 0 & 0 & 0 & 0 \\
0.20 & 0.26 &  0 & 0 & 0.06 & 0 & 0.48 & 0 & 0 & 0 & 0 & 0 & 0 & 0 & 0 & 0 & 0 & 0 & 0 & 0 \\
0 & 0 &  0 & 0.21 & 0 & 0 & 0 & 0.52 & 0.27 & 0 & 0 & 0 & 0 & 0 & 0 & 0 & 0 & 0 & 0 & 0 \\
0 & 0.15 &  0 & 0 & 0 & 0 & 0 & 0.27 & 0.39 & 0 & 0 & 0 & 0 & 0 & 0 & 0 & 0.19 & 0 & 0 & 0 \\
0 & 0 &  0 & 0 & 0.07 & 0 & 0 & 0 & 0 & 0.46 & 0 & 0 & 0 & 0 & 0 & 0.22 & 0 & 0 & 0.25 & 0 \\
0 & 0 &  0.25 & 0 & 0 & 0 & 0 & 0 & 0 & 0 & 0.56 & 0 & 0 & 0 & 0 & 0.14 & 0.05 & 0 & 0 & 0 \\
0 & 0 &  0 & 0 & 0 & 0.28 & 0 & 0 & 0 & 0 & 0 & 0.53 & 0.19 & 0 & 0 & 0 & 0 & 0 & 0 & 0 \\
0 & 0 &  0 & 0.13 & 0 & 0 & 0 & 0 & 0 & 0 & 0 & 0.19 & 0.55 & 0 & 0 & 0.13 & 0 & 0 & 0 & 0 \\
0 & 0 &  0.07 & 0 & 0 & 0.22 & 0 & 0 & 0 & 0 & 0 & 0 & 0 & 0.55 & 0 & 0 & 0 & 0 & 0 & 0.16 \\
0 & 0 &  0 & 0.11 & 0 & 0 & 0 & 0 & 0 & 0 & 0 & 0 & 0 & 0 & 0.68 & 0.21 & 0 & 0 & 0 & 0 \\
0 & 0 &  0 & 0 & 0 & 0 & 0 & 0 & 0 & 0.22 & 0.14 & 0 & 0.13 & 0 & 0.21 & 0.30 & 0 & 0 & 0 & 0 \\
0 & 0 &  0 & 0 & 0 & 0 & 0 & 0 & 0.19 & 0 & 0.05 & 0 & 0 & 0 & 0 & 0 & 0.57 & 0 & 0 & 0.19 \\
0.12 & 0 &  0 & 0 & 0 & 0 & 0 & 0 & 0 & 0 & 0 & 0 & 0 & 0 & 0 & 0 & 0 & 0.53 & 0.18 & 0.17 \\
0 & 0 &  0 & 0 & 0 & 0 & 0 & 0 & 0 & 0.25 & 0 & 0 & 0 & 0 & 0 & 0 & 0 & 0.18 & 0.57 & 0 \\
0.11 & 0 &  0 & 0 & 0 & 0 & 0 & 0 & 0 & 0 & 0 & 0 & 0 & 0.16 & 0 & 0 & 0.19 & 0.17 & 0 & 0.37
\end{array} \right]. }
\end{eqnarray*}




\end{document}